\documentclass[11pt]{article}
\usepackage[a4paper,margin=1in]{geometry}
\usepackage[nottoc,numbib]{tocbibind}
\usepackage{lmodern}
\usepackage[dvipsnames]{xcolor}

\usepackage{amsmath, amssymb, amsthm, mathtools}
\usepackage{bm, mathrsfs, dsfont}
\allowdisplaybreaks

\usepackage{graphicx, booktabs, float}

\usepackage{authblk}
\usepackage[round]{natbib}
\usepackage[hyphens]{url}
\usepackage{fancyhdr}

\usepackage{hyperref}
\hypersetup{colorlinks=true, citecolor=Plum, linkcolor=PineGreen}
\usepackage{cleveref}

\theoremstyle{plain}
\newtheorem{theorem}{Theorem}[section]

\newtheorem{lemma}{Lemma}[section]
\newtheorem{corollary}{Corollary}[section]
\newtheorem{proposition}{Proposition}[section]

\theoremstyle{definition}
\newtheorem{remark}{Remark}[section]

\NewDocumentCommand{\refLBNN}{}{\ref{sec:lbnn}}
\NewDocumentCommand{\refTempering}{}{\ref{sec:likelihoodtempering}}

\newcommand{\R}{\mathbb R}

\newcommand{\E}{\mathbb E}
\DeclareMathOperator{\Var}{Var}

\DeclareMathOperator{\tr}{tr}

\renewcommand{\epsilon}{\varepsilon}

\newcommand{\iidsim}{\stackrel{\text{iid}}{\sim}}
\newcommand{\ind}{\stackrel{d}{\to}}
\newcommand{\inp}{\stackrel{p}{\to}}

\title{The Impact of Likelihood Tempering on the Limiting Predictive Moments of Variational Bayesian Linear Neural Networks}
\author[,1,2]{Ian Zhang\thanks{Corresponding author: \texttt{ianz.zhang@mail.utoronto.ca}}}
\author[1]{Thibault Randrianarisoa}
\affil[1]{Department of Statistical Sciences, University of Toronto}
\affil[2]{Dunlap Institute for Astronomy and Astrophysics}

\date{\today}

\begin{document}
\setcounter{footnote}{1}
\maketitle

\begin{abstract}
    In wide Bayesian neural networks, Gaussian mean-field variational inference is prone to \textit{prior dominance}: the Kullback-Leibler (KL) regularization term of the ELBO outweighs the expected log-likelihood, and the variational predictive distribution collapses to the prior predictive as the width $M$ grows. Tempering the likelihood, by raising it to the power $1/T$ for a temperature $T < 1$, is equivalent to scaling the KL term by $T$. We ask in this paper how fast $T$ must decrease with $M$ to counteract this degeneracy and strike a good balance between the two terms. For single-hidden-layer linear networks with isotropic Gaussian priors, we derive the limiting predictive distribution under schedules of the form $T = \tau/M^{c}$, with constants $\tau, c > 0$, as $M \to \infty$ and compare it with the untempered neural network Gaussian process (NNGP) posterior, the infinite-width limit of the exact posterior. Our main result is that the predictive expectation and variance undergo phase transitions at different scales: the limiting expectation leaves its prior value at $c = 1/2$, once $\tau$ falls below an explicit threshold, and equals the least-squares prediction for $c > 1/2$, whereas the limiting variance keeps its prior value for $c < 1$, matches the NNGP's for $c=1$, and vanishes for $c > 1$. With suitable choices of $\tau,c$, one can recover either the NNGP posterior expectation or its variance.
\end{abstract}

\section{Introduction}\label{sec:introduction}
 
Bayesian inference equips a model with a posterior distribution over its parameters, from which one obtains uncertainty quantification and a principled means of comparing models through the evidence \citep{MacKay1992,Ghahramani2015}. Neural networks have been treated in this way since \citet{MacKay1992} and \citet{nealPriorsInfiniteNetworks1996}: a Bayesian neural network (BNN) places a prior $p(\theta)$ on the weights, separates the uncertainty due to observation noise from the uncertainty about the fitted function \citep{kendallWhatUncertaintiesWe2017}, and thus combines the flexibility of neural networks with principled uncertainty quantification \citep{jospinHandsOnBayesianNeural2022,Arbel2023}. The posterior $p(\theta\mid\mathcal{D})$ of a BNN is intractable, however, and sampling methods such as MCMC, although asymptotically exact and successfully applied in practice \citep{nealPriorsInfiniteNetworks1996,Izmailov2021}, remain costly for modern architectures. Variational inference (VI) offers a tractable alternative, approximating the posterior by the member $q^*(\theta)$ of a chosen family $\mathcal{Q}$ that maximizes the evidence lower bound (ELBO) \citep{bleiVariationalInferenceReview2017}. VI became particularly practical once stochastic optimization of the ELBO became available \citep{Hoffman2013,Graves2011}, and with the reparameterization trick \citep{KingmaWelling2014} it resembles ordinary backpropagation \citep{Blundell2015}. Gaussian variational learning has been applied successfully to large modern networks \citep{Shen2024}.
 
\paragraph{Prior dominance.}
The Gaussian mean-field variational family, under which (blocks of) parameters are independent, is the most common variational family for BNNs, and is commonly held to underestimate posterior uncertainty \citep{Minka2005,bleiVariationalInferenceReview2017}. In BNNs, it also restricts the predictive uncertainty that can be expressed \citep{foongExpressivenessApproximateInference2020}. While it is true that underestimation does hold in the parameter space \citep{MargossianSaul2023}, this is not (completely) so for predictions. In conjugate Bayesian linear regression, \citet{odgersGaussianMeanField2026} show that the mean-field optimum overestimates the predictive variance on inputs resembling the training data and underestimates it in other directions. In overparameterized random feature regression with an isotropic Gaussian variational family, the optimal variational posterior variance likewise reverts to its prior value as the number of features grows with the sample size fixed \citep{harveyLearningHyperparametersDataEmphasized2026}. In both models the parameters enter the predictor linearly and the mean is preserved, so the degradation concerns the variance alone. In neural networks, where the parameters of consecutive layers combine in a nonlinear fashion, the mean is lost as well. \citet{cokerWideMeanFieldBayesian2022a} prove that, for networks with odd, Lipschitz activations, the optimal Gaussian mean-field predictive distribution converges to the prior predictive as the architecture width increases, so that the approximation effectively ignores the data. Empirical signs of this degeneracy appear in \citet{trippe2018overpruningvariationalbayesianneural}, where weights revert to their prior, and in \citet{sunFUNCTIONALVARIATIONALBAYESIAN2018}, where predictions deteriorate as networks grow. We refer to this phenomenon as ``prior dominance'' \citep[see also][]{huixVariationalInferenceOverparameterized2022,Descours2023}: the Kullback-Leibler divergence (KL) term of the ELBO outweighs the expected log-likelihood, thus the variational posterior reverts to the prior in function-space.
 
\paragraph{Remedies.}
Given the previous observation, two families of approaches have been proposed. Since network outputs are of more interest than network parameters themselves, the first approach moves the divergence to function-space, where the induced prior on the functions computed by the network is considered \citep{sunFUNCTIONALVARIATIONALBAYESIAN2018,rudnerTractableFunctionSpaceVariational2022}. Indeed, when the variational distribution and the prior over functions are both induced by distributions over the parameters, the data-processing inequality ensures that the functional KL-divergence stays below its parameter-space counterpart, reducing the regularization. However, no such bound is available when the prior is specified directly on functions (for instance, via a Gaussian process), and functional KL-divergence may be infinite \citep{sunFUNCTIONALVARIATIONALBAYESIAN2018,Burt2021}. This motivates regularized variants \citep{CinquinBamler2025}.

The second approach keeps the parameter-space objective but downweights its KL regularization term relative to the expected log-likelihood \citep{zhangNoisyNaturalGradient2018,osawaPracticalDeepLearning2019}. This reweighting is equivalent to tempering the likelihood alone, that is, to raising it to the power $1/T$ for a temperature $T<1$, and recent work prescribes how $T$ should depend on the size of the model instead of tuning it: \citet{huixVariationalInferenceOverparameterized2022} and \citet{Descours2023} scale the KL term inversely with the number of neurons of a shallow network, while \citet{harveyLearningHyperparametersDataEmphasized2026} upweight the likelihood by the ratio of the number of parameters to the sample size. Likelihood tempering must be distinguished from the `cold posteriors' of \citet{wenzelHowGoodBayes2020}, which raise the whole posterior, including the prior, to the power $1/T$ \citep{Aitchison2021,odgersGaussianMeanField2026}. \citet{Arbel2023} review the debate surrounding both. \citet{odgersGaussianMeanField2026} argue that, in conjugate linear regression, a cold posterior corrects the predictive variance while keeping the exact mean, whereas likelihood tempering cannot correct the variance without displacing the mean. When the mean itself has been lost to the prior, it is not clear which quantities tempering can restore.
 
\paragraph{Two infinite-width scalings.}
These prescriptions of temperatures depend on the parameterization of the network. Indeed, different scalings of a fully-connected single-hidden-layer network with its width $M$ lead to different limits. In the mean-field scaling, the output is the average of $M$ neurons, $f(x) = M^{-1}\sum_{m=1}^{M} b_m\,\varphi(a_m^\top x)$, whose priors do not depend on $M$. It is the scaling of the mean-field analyses of gradient training \citep{Mei2018,ChizatBach2018,RotskoffVandenEijnden2022,SirignanoSpiliopoulos2020}, and it is the one adopted by \citet{huixVariationalInferenceOverparameterized2022} and \citet{Descours2023}. By the law of large numbers, the prior over functions concentrates on a deterministic function, losing the benefits of the Bayesian approach in the limit $M\to\infty$. In the scaling we adopt here, the sum is instead normalized by $M^{-1/2}$, as in the neural tangent kernel (NTK) parameterization \citep{jacotNeuralTangentKernel2018} or, equivalently for the prior, by giving the output weights variance of order $1/M$. The induced prior over functions then converges to a non-degenerate Gaussian process, the neural network Gaussian process (NNGP) \citep{nealPriorsInfiniteNetworks1996,leeDeepNeuralNetworks2018,matthews2018gaussianprocessbehaviourwide}, and the exact posterior converges to the corresponding NNGP posterior \citep{hronExactPosteriorDistributions2020}, whose mean and variance are available in closed form under a Gaussian likelihood. While this limit provides important insights into the behaviour of BNNs, it is also the natural regime in which to ask whether a variational approximation recovers posterior uncertainty. \citet{cokerWideMeanFieldBayesian2022a} establish under this setting the collapse of the mean-field approximation, and the random-feature model of \citet{harveyLearningHyperparametersDataEmphasized2026} uses the same normalization.
 
In our notation, the prescriptions of the previous part all correspond to temperatures of order $1/M$. In the mean-field scaling, this is the rate at which the two terms of the ELBO remain balanced, and \citet{Descours2023} prove a law of large numbers for the resulting training dynamics. In the random-feature model, where the mean is preserved, it prevents the variational variance from reverting to its prior value \citep{harveyLearningHyperparametersDataEmphasized2026}. Neither result covers the $M^{-1/2}$ scaling with both layers trained, in which the predictive mean is lost as well.
 
\paragraph{Question and setting.}
We ask how fast the temperature must decrease with the width to counteract prior dominance in the NNGP regime, and whether a single schedule can recover both the posterior predictive mean and the posterior predictive variance. We study both questions in the tractable setting of single-hidden-layer linear networks in the $M^{-1/2}$ scaling, with isotropic zero-mean Gaussian priors, and in the limit $M \to \infty$ with the sample size $n$ and input dimension $d$ fixed. Linear networks have long served as a tractable alternative to their nonlinear counterparts \citep{Saxe2014}, and their mean-field variational approximations have been analyzed in detail at fixed width, where the linear network amounts to reduced-rank regression or matrix factorization \citep{nakajimaVariationalBayesSolution2007,nakajimaTheoreticalAnalysisBayesian2011,nakajimaGlobalAnalyticSolution2013}. Our analysis departs from these works in that we let the width grow to infinity, beyond the reduced-rank regime, we temper the likelihood, and we focus on the resulting predictive distribution over functions. We consider schedules of the form $T = \tau/M^{c}$ with constants $\tau, c > 0$, and we compare the limiting variational predictive distribution with the exact untempered NNGP posterior.
 
\paragraph{Contributions.}
Our contributions are the following.
\begin{enumerate}
    \item We derive the limiting variational predictive distribution under every schedule $T = \tau/M^{c}$ (Theorem~\ref{thm:main-theorem}). The predictive mean and variance undergo phase transitions at different scales: the limiting mean leaves its prior value at $c = 1/2$, once $\tau$ falls below an explicit threshold, and equals the least-squares prediction for $c > 1/2$, whereas the limiting variance transitions at $c=1$.
    \item We show that suitable choices of $\tau$ recover either the NNGP posterior mean, at $c = 1/2$, or the NNGP posterior variance, at $c = 1$, and that no schedule of this form recovers both (Corollary \ref{cor:recovery}). In particular, temperatures of order $1/M$, as previously prescribed, recover at best the variance but with an expectation that already equals the least-squares prediction.
\end{enumerate}

\section{Setting}\label{sec:lbnn}
This section introduces BNNs and their infinite-width limit, the mean-field approximation that we study and the structure of its global minimizers, and it shows that, without tempering, this approximation reverts to the prior. General definitions of deep Bayesian neural networks and of variational inference are recalled in Appendix~\ref{app:background-details}.

\subsection{Model and Neural Network Gaussian Process}\label{sec:nngp}
We consider a network with one hidden layer of width $M$, identity activation and scalar output,
\begin{equation}\label{eq:lnn}
    f_\theta(x) = \frac{1}{\sqrt M}W_2\,\frac{1}{\sqrt d}W_1x = BA^\top x = \sum_{m=1}^M b_m a_m^\top x,
\end{equation}
where $x\in\R^d$, $W_1\in\R^{M\times d}$, $W_2\in\R^{1\times M}$ and $\theta=(A,B)$. The matrix $A=W_1^\top/\sqrt d\in\R^{d\times M}$ has columns $a_m\in\R^d$, the vector $B=W_2/\sqrt M\in\R^{1\times M}$ has entries $b_m\in\R$, and each index $m$ is a hidden neuron with input weights $a_m$ and output weight $b_m$. The prior and the observation model are
\begin{align}
    \operatorname{vec}(A) &\sim \mathcal N_{Md}\bigl(0,\, d^{-1} I_{Md}\bigr), \label{eq:prior-A}\\
    B^\top &\sim \mathcal N_M\bigl(0,\, M^{-1} I_M\bigr), \label{eq:prior-B}\\
    y \mid x,A,B &\sim \mathcal N\bigl(BA^\top x,\, \sigma^2\bigr), \label{eq:model-equation}
\end{align}
with prior density $ p (A,B)= p (A) p (B)$ and known noise variance $\sigma^2>0$.

The factors $1/\sqrt d$ and $1/\sqrt M$ in \eqref{eq:lnn} are those of the NTK parameterization \citep{jacotNeuralTangentKernel2018}, in which $W_1$ and $W_2$ have standard Gaussian entries, whereas \eqref{eq:prior-A} and \eqref{eq:prior-B} express the same prior in the standard parameterization of $(A,B)$. The two descriptions lead to the same variational problem below, since $(W_1,W_2)\mapsto(A,B)$ is a fixed linear bijection that preserves Gaussian mean-field families and leaves the Kullback-Leibler (KL) divergence invariant. We therefore work with $(A,B)$ throughout.

Under this prior, $f_\theta(x)$ has mean zero and variance $x^\top x/d$ at every width. Conditionally on $A$, it is Gaussian with variance $\|A^\top x\|_2^2/M$, which converges to $x^\top x/d$ by the law of large numbers, whence the following limit (Appendix~\ref{app:proofdistBAxprior}).
\begin{proposition}\label{thm:distBAxprior}
    Under \eqref{eq:prior-A} and \eqref{eq:prior-B}, for each fixed $x\in\R^d$, $f_\theta(x)\ind\mathcal N(0,x^\top x/d)$ as $M\to\infty$.
\end{proposition}
Proposition~\ref{thm:distBAxprior} is the one-dimensional marginal of the neural network Gaussian process (NNGP) limit \citep{nealPriorsInfiniteNetworks1996,leeDeepNeuralNetworks2018,matthews2018gaussianprocessbehaviourwide}, which here is the centred Gaussian process with covariance function $x^\top y/d$, mean and variance
\begin{equation}\label{eq:nngp-prior}
    m_0(x)=0, \qquad v_0(x)=\frac{x^\top x}{d}.
\end{equation}

For $n$ observations, write $X^n=\{x_1,\ldots,x_n\}$ and $Y^n=\{y_1,\ldots,y_n\}$, set $\alpha=n/\sigma^2$ and
\begin{equation*}
    Q=\frac1n\sum_{i=1}^n x_ix_i^\top\in\R^{d\times d}, \quad R=\frac1n\sum_{i=1}^n y_ix_i^\top\in\R^{1\times d}.
\end{equation*}
Since the covariance function is linear, the NNGP posterior with the Gaussian likelihood is that of Bayesian linear regression with prior $\mathcal N(0,I_d/d)$ on the coefficients \citep[Section~2.1]{rasmussenGaussianProcessesMachine2006}, and the posterior mean and variance of the latent output are
\begin{equation}\label{eq:nngp-posterior}
    \begin{aligned}
        m_{\text{NNGP}}(x) &= \alpha R(dI_d+\alpha Q)^{-1}x,\\
        v_{\text{NNGP}}(x) &= x^\top(dI_d+\alpha Q)^{-1}x.
    \end{aligned}
\end{equation}
For fixed data, the exact BNN posterior converges to this NNGP posterior as $M\to\infty$ \citep{hronExactPosteriorDistributions2020}, which makes \eqref{eq:nngp-posterior} the natural benchmark for the variational approximations below. All width limits are taken with $n$, $d$, the data and the test input $x$ fixed. Predictive quantities refer to the latent output $f_\theta(x)$, and $\sigma^2$ must be added to the variance when predicting a new observation.

Throughout the paper, we assume that $Q=I_d$. Then $Rx$ is the least-squares prediction at $x$, and \eqref{eq:nngp-posterior} reduces to 
\begin{align*}
m_{\text{NNGP}}(x)&=\alpha Rx/(\alpha+d),\\
v_{\text{NNGP}}(x)&=x^\top x/(\alpha+d).
\end{align*}
Any design of full column rank can be brought to this form by a linear transformation of the inputs, at the price of the data-dependent prior $a_m\sim\mathcal N(0,Q^{-1}/d)$ on the original weights (Remark~\ref{rem:whitening} in Appendix~\ref{app:supplementary-LBNN}).

\subsection{Mean-Field Variational Inference}\label{sec:lnnvi}\label{sec:variational-optimization}
Variational inference approximates the posterior $\pi(A,B\mid X^n,Y^n)\propto p (A,B)\,p(Y^n\mid X^n,A,B)$ by the member $q^*$ of a family $\mathcal{Q}$ that minimizes $\text{KL}(q\,\|\,\pi(\cdot\mid X^n,Y^n))$ or, equivalently, maximizes the evidence lower bound $\text{ELBO}(q)=\E_q[\log p(Y^n\mid X^n,A,B)]-\text{KL}(q\,\|\, p )$ \citep{bleiVariationalInferenceReview2017}. We use the Gaussian mean-field family $\mathcal{Q}_{\text{MF}}$ of densities
\begin{equation}\label{eq:vbfactorization}
    q(A,B)=\prod_{m=1}^M q(a_m)\,q(b_m),
\end{equation}
where $q(a_m)=\mathcal N_d(\mu_{a_m},\Sigma_{a_m})$ with $\Sigma_{a_m}\succ0$ and $q(b_m)=\mathcal N(\mu_{b_m},\Sigma_{b_m})$ with $\Sigma_{b_m}>0$. This family makes the neurons independent, as well as the input and output weights of each neuron, but allows a full covariance matrix for each $a_m$, as in the SimpleVB approximation of \citet{nakajimaGlobalAnalyticSolution2013}. We call neuron $m$ \emph{active} under $q$ if $(\mu_{a_m},\mu_{b_m})\neq(0,0)$.

Following \citet{nakajimaTheoreticalAnalysisBayesian2011}, we work with the variational Bayes free energy
\begin{equation}\label{eq:vbfreeenergyoriginal}
    F_{\text{VB}}(q)=\E_q\left[\log\frac{q(A,B)}{p(Y^n,A,B\mid X^n)}\right]=-\text{ELBO}(q)
\end{equation}
and minimize $\mathcal L^{\text{VB}}(q)=2F_{\text{VB}}(q)+C$ over $\mathcal{Q}_{\text{MF}}$, where $C$ does not depend on $q\in \mathcal{Q}_{\text{MF}}$. One finds (Appendix~\ref{app:longcomputationLVB})
\begin{equation}\label{eq:minproblemexpanded}
\begin{aligned}
    &\mathcal L^{\text{VB}}(q) = \sum_{m=1}^M \Bigl[d\bigl(\|\mu_{a_m}\|_2^2 + \tr(\Sigma_{a_m})\bigr) - \log|\Sigma_{a_m}| + M\bigl(\mu_{b_m}^2 + \Sigma_{b_m}\bigr) - \log\Sigma_{b_m}\Bigr] \\
    &\quad + \frac{1}{\sigma^2}\sum_{i=1}^n \Bigl[\Bigl(y_i - \sum_{m=1}^M \mu_{b_m}\mu_{a_m}^\top x_i\Bigr)^2 + \sum_{m=1}^M \Bigl(\mu_{b_m}^2 x_i^\top\Sigma_{a_m}x_i + \Sigma_{b_m}(\mu_{a_m}^\top x_i)^2 + \Sigma_{b_m}x_i^\top\Sigma_{a_m}x_i\Bigr)\Bigr].
\end{aligned}
\end{equation}

\subsection{Structure of the Global Minimizers}\label{sec:global-minimizers}

For fixed covariances, the means enter the data-fit term of \eqref{eq:minproblemexpanded} only through $w=\sum_{m=1}^M\mu_{b_m}\mu_{a_m}$, while the remaining terms form a quadratic penalty on the means that is separable across neurons. By the inequality between arithmetic and geometric means, for a given $w$ this penalty is smallest when $w$ is carried by a single neuron, and the covariance equation \eqref{posterior cov a_m} below then excludes every other minimizer (Appendix~\ref{app:orthogexpecAB}). 

Theorem~\ref{thm:orthogexpecAB} below thus guarantees an exact form of the `overpruning' of mean-field BNNs observed by \citet{trippe2018overpruningvariationalbayesianneural}: all neurons but one revert to their prior mean.

\begin{theorem}\label{thm:orthogexpecAB}
    Assume $Q=I_d$. The objective \eqref{eq:minproblemexpanded} admits a global minimizer over $\mathcal{Q}_{\text{MF}}$, and every global minimizer satisfies $(\mu_{a_m},\mu_{b_m})=(0,0)$ for all but at most one neuron $m\in\{1,\ldots,M\}$.
\end{theorem}

\begin{remark}\label{rem:tempered-one-active}
    Theorem~\ref{thm:orthogexpecAB} also holds for the tempered objective of Section~\ref{sec:likelihoodtempering}, in which $\sigma^2$ is replaced by $T\sigma^2$ (Appendix~\ref{app:tempdownweight}).
\end{remark}

Since the parameter domain is open, global minimizers are stationary points of \eqref{eq:minproblemexpanded}. At a stationary point with at most one active neuron, the first-order conditions read, for every $m$ (Appendix~\ref{app:coordinateupdates}),
\begin{align}
    \mu_{a_m} &= \alpha\mu_{b_m}\Sigma_{a_m}R^\top, \label{posterior mean a_m}\\
    \Sigma_{a_m} &= \frac{I_d}{d+\alpha\left[\mu_{b_m}^2+\Sigma_{b_m}\right]}, \label{posterior cov a_m}\\
    \mu_{b_m} &= \alpha\Sigma_{b_m}R\mu_{a_m}, \label{posterior mean b_m}\\
    \Sigma_{b_m} &= \Bigl(M+\alpha\bigl[\|\mu_{a_m}\|_2^2+\tr(\Sigma_{a_m})\bigr]\Bigr)^{-1}. \label{posterior cov b_m}
\end{align}

These are the stationary conditions of the SimpleVB approximation \citep{nakajimaGlobalAnalyticSolution2013}, adapted to our priors and to $Q=I_d$. We note that in general, $R$ must be replaced in \eqref{posterior mean a_m} and \eqref{posterior mean b_m} by the residual $R-\sum_{k\neq m}\mu_{b_k}\mu_{a_k}^\top$, but the correction vanishes at every global minimiser, by Theorem \ref{thm:orthogexpecAB}.

In particular, every $\Sigma_{a_m}$ is proportional to $I_d$ at a stationary point, so that every global minimizer over $\mathcal{Q}_{\text{MF}}$ belongs to the fully factorized Gaussian family, which is contained in $\mathcal{Q}_{\text{MF}}$. Both families therefore have the same global minimizers, and our results apply verbatim to fully factorized mean-field VI. The full covariance of $q(a_m)$ matters only for general designs, through Remark~\ref{rem:whitening}.

If neuron $j$ is the only possibly active neuron, the independence of the factors gives the predictive moments
\begin{align*}
    \E_q(BA^\top x) = \mu_{b_j}\mu_{a_j}^\top x, \qquad
    \Var_q(BA^\top x) = \mu_{b_j}^2x^\top\Sigma_{a_j}x+\Sigma_{b_j}(\mu_{a_j}^\top x)^2 +\sum_{m=1}^M\Sigma_{b_m}x^\top\Sigma_{a_m}x.
\end{align*}
The mean is carried by a single neuron, whereas all $M$ neurons contribute to the variance, so that the analysis of the predictive distribution requires both the active neuron and the covariances of the inactive ones.

\subsection{Prior Dominance}\label{sec:limitbehaviour}

\begin{lemma}\label{lem:limitdistofbhahxlemma1}
    Assume $Q=I_d$ and let $q$ be a global minimizer of \eqref{eq:minproblemexpanded}. If $M>\alpha^2\|R\|_2^2/d$, then no neuron is active under $q$.
\end{lemma}

Beyond the threshold of Lemma \ref{lem:limitdistofbhahxlemma1}, all means therefore vanish, the covariances of the inactive neurons approach their prior values, and the central limit theorem returns the NNGP prior of Proposition~\ref{thm:distBAxprior}. 

\begin{theorem}\label{thm:limitdistofbhahx}
    Assume $Q=I_d$ and, for each $M$, let $q_M$ be a global minimizer of \eqref{eq:minproblemexpanded}. For each fixed $x\in\R^d$, the law of $BA^\top x$ under $q_M$ converges weakly to $\mathcal N(0,x^\top x/d)$ as $M\to\infty$.
\end{theorem}

The proofs of Lemma~\ref{lem:limitdistofbhahxlemma1} and Theorem~\ref{thm:limitdistofbhahx} are given in Appendices~\ref{app:limitdistofbhahxlemma1} and~\ref{app:limitdistofbhahxproof}.

Theorem~\ref{thm:limitdistofbhahx} is the shallow linear case of Theorem~1 of \citet{cokerWideMeanFieldBayesian2022a}, which covers odd Lipschitz activations and thus the identity. Appendix~J of that paper also treats the predictive mean of linear networks. The next section studies how likelihood tempering changes this limit.

\section{Likelihood Tempering}\label{sec:likelihoodtempering}

\subsection{Tempered Variational Objective}\label{sec:tempered-objective}
For a temperature $T>0$, the tempered posterior raises the likelihood to the power $1/T$,
\begin{equation*}
    \pi_T(A,B\mid X^n,Y^n)\propto p(A,B)\,p(Y^n\mid X^n,A,B)^{1/T},
\end{equation*}
a form of generalized Bayesian updating \citep{bissiriGeneralFrameworkUpdating2016a}. With the Gaussian likelihood \eqref{eq:model-equation}, 
\begin{equation*}
\pi_T\propto p(A,B)\exp\{-S(A,B)/(2T\sigma^2)\}, \qquad S(A,B)=\sum_{i=1}^n(y_i-BA^\top x_i)^2,
\end{equation*}
where $S(A,B)$ is the residual sum of squares, so that $\pi_T$ is the exact posterior of the same model with noise variance $T\sigma^2$. Only the likelihood is tempered, unlike in the cold posteriors of \citet{wenzelHowGoodBayes2020}, where the prior is raised to the same power. We call the tempering \emph{cold} when $T<1$, and the data then receive more weight than in the untempered posterior.

The tempered variational approximation $q_T^*$ is a global minimizer of $\text{KL}(q\,\|\,\pi_T(\cdot\mid X^n,Y^n))$ over $\mathcal{Q}_{\text{MF}}$, or equivalently of
\begin{equation}\label{eq:tempered-elbo}
    T\,\text{KL}(q\,\|\,p)-\E_q[\log p(Y^n\mid X^n,A,B)],
\end{equation}
so that tempering down-weights the KL term of the ELBO by the factor $T$ \citep{yang$alpha$variationalInference2020}. Up to a positive factor and an additive constant, \eqref{eq:tempered-elbo} is the objective \eqref{eq:minproblemexpanded} with $\sigma^2$ replaced by $T\sigma^2$ (Appendix~\ref{app:tempdownweight}). Theorem~\ref{thm:orthogexpecAB} and the first-order conditions \eqref{posterior mean a_m} to \eqref{posterior cov b_m} therefore hold for $q_T^*$ with $\sigma^2$ replaced by $T\sigma^2$ (Remark~\ref{rem:tempered-one-active}).

We let the temperature decrease with the width as $T_M=\tau/M^c$, with constants $\tau,c>0$, and we write $q_{T_M}^*$ for a global minimizer at width $M$, with $\mu_{a_m}$, $\Sigma_{a_m}$, $\mu_{b_m}$ and $\Sigma_{b_m}$ for its moments, without indicating their dependence on $M$.

\subsection{Phase Transitions}\label{sec:phase-transitions}
The means and the covariances of the variational factors respond to tempering at different scales. For the means, the comparison detailed in Appendix~\ref{app:means-threshold} gives an explicit threshold.

\begin{lemma}\label{lem:tempered-zero-means}
    Assume $Q=I_d$ and let $\tau_*=\alpha\|R\|_2/\sqrt d$. For every width $M$ and every temperature $T\geq\tau_*/\sqrt M$, all the means of every global minimizer of the tempered objective vanish.
\end{lemma}

Fix the covariances and write $w=\sum_m\mu_{b_m}\mu_{a_m}$, as in Section~\ref{sec:global-minimizers}. Relative to zero means, the KL term of \eqref{eq:tempered-elbo} increases by
\begin{equation*}
    \frac T2\sum_{m=1}^M\bigl(d\|\mu_{a_m}\|_2^2+M\mu_{b_m}^2\bigr)\geq T\sqrt{dM}\,\|w\|_2,
\end{equation*}
by the AM-GM inequality and the triangle inequality, with equality for a single active neuron with $d\|\mu_{a_m}\|_2^2=M\mu_{b_m}^2$ (Theorem~\ref{thm:orthogAB}). The expected negative log-likelihood decreases by at most $\alpha\|R\|_2\|w\|_2$, by the Cauchy-Schwarz inequality. The two bounds balance at $T=\tau_*/\sqrt M$, which is the threshold of Lemma~\ref{lem:tempered-zero-means}. For $T=1$, Lemma~\ref{lem:tempered-zero-means} is Lemma~\ref{lem:limitdistofbhahxlemma1}. Under the schedule $T_M=\tau/M^c$, it shows that the predictive mean can depart from zero only if $c>1/2$, or if $c=1/2$ and $\tau<\tau_*$.

For the covariances, consider an inactive neuron and write $s_m^{(M)}=M\Sigma_{b_m}$. With $\sigma^2$ replaced by $T_M\sigma^2$, \eqref{posterior cov a_m} and \eqref{posterior cov b_m} become
\begin{align*}
    \Sigma_{a_m} &= \frac{I_d}{d+\frac{\alpha}{MT_M}s_m^{(M)}},\\
    s_m^{(M)} &= \Bigl(1+\frac{\alpha}{MT_M}\tr(\Sigma_{a_m})\Bigr)^{-1},
\end{align*}
which involve the temperature only through $\alpha/(MT_M)=(\alpha/\tau)M^{c-1}$. For $c<1$, the inactive neurons therefore keep their prior covariances in the limit, $s_m^{(M)}\to1$ and $\Sigma_{a_m}\to I_d/d$, and their total contribution to the predictive variance converges to $v_0(x)$. The active neuron, if any, contributes a variance of at most $2T_Mx^\top x/\alpha$ (Appendix~\ref{app:proof-phase-transition}). The predictive variance can thus only change if $c\geq1$. Both necessary conditions are sharp, and the next theorem identifies the limit in every regime.

\begin{theorem}\label{thm:main-theorem}
    Assume that $Q=I_d$ and $R\neq0$, let $T_M=\tau/M^c$ with $\tau,c>0$, and let $\tau_*$ be as in Lemma~\ref{lem:tempered-zero-means}. Let $s$ be the unique solution in $(0,1)$ of
    \begin{equation}\label{eq:s-fixed-point}
        s=\Bigl(1+\frac{\alpha d}{\alpha s+\tau d}\Bigr)^{-1},
    \end{equation}
    and set $m_\infty(x)=Rx$, $m_{1/2}(x)=(1-\tau/\tau_*)Rx$ and $v_1(x)=x^\top x/(\alpha/\tau+d/s)$. For each fixed $x\in\R^d$ and any choice of global minimizers $q_{T_M}^*$, the law of $BA^\top x$ under $q_{T_M}^*$ converges weakly, as $M\to\infty$, to the following limits.
    \begin{enumerate}
    \item If $c<1/2$, to the prior $\mathcal N(0,v_0(x))$.
    \item If $c=1/2$:
    \begin{enumerate}
        \item to the prior $\mathcal N(0,v_0(x))$ if $\tau\geq\tau_*$,
        \item to $\mathcal N(m_{1/2}(x),v_0(x))$ if $\tau<\tau_*$.
    \end{enumerate}
    \item If $1/2<c<1$, to $\mathcal N(m_\infty(x),v_0(x))$.
    \item If $c=1$, to $\mathcal N(m_\infty(x),v_1(x))$.
    \item If $c>1$, to the point mass at $m_\infty(x)$.
\end{enumerate}
\end{theorem}

By the time the variance responds, the mean has already passed the NNGP posterior mean and reached the least-squares prediction. The fixed point \eqref{eq:s-fixed-point} solves a quadratic equation, whose root in $(0,1)$ is
\begin{equation*}
    s=\frac{2\tau d}{\tau d+\alpha(d-1)+\sqrt{(\tau d+\alpha(d-1))^2+4\alpha\tau d}},
\end{equation*}
and $v_1(x)$ increases from $0$ to $v_0(x)$ as $\tau$ ranges over $(0,\infty)$. A proof of Theorem~\ref{thm:main-theorem} is given in Appendix~\ref{app:proof-phase-transition}.

For comparison, tempering the NNGP posterior itself replaces $\alpha$ by $\alpha/T_M$ in \eqref{eq:nngp-posterior}, so that it converges to the point mass at $m_\infty(x)$ for every $c>0$, a limit that the variational predictive reaches only for $c>1$. The recovery of the untempered NNGP posterior moments is characterized as follows.

\begin{corollary}\label{cor:recovery}
    Under the assumptions of Theorem~\ref{thm:main-theorem}, the following holds.
    \begin{enumerate}
        \item For $c=1/2$, the limiting mean equals $m_{\text{NNGP}}(x)$ for all $x$ if and only if $\tau=\tau_\mu$, where $\tau_\mu=\tau_*d/(\alpha+d)$.
        \item For $c=1$, the limiting variance equals $v_{\text{NNGP}}(x)$ for all $x$ if and only if $\tau=\tau_v$, where
        \begin{equation*}
            \tau_v=\frac{d+1+\sqrt{(d+1)^2-4d\alpha/(d+\alpha)}}{2}\in(1,d+1).
        \end{equation*}
        \item No schedule $T_M=\tau/M^c$ recovers both $m_{\text{NNGP}}(x)$ and $v_{\text{NNGP}}(x)$ for all $x$.
    \end{enumerate}
\end{corollary}
The first item solves $1-\tau/\tau_*=\alpha/(\alpha+d)$. For the second, $v_1=v_{\text{NNGP}}$ means $\alpha/\tau+d/s=\alpha+d$, which, inserted in \eqref{eq:s-fixed-point}, gives $\tau^2-(d+1)\tau+d\alpha/(d+\alpha)=0$. The smaller root of this equation lies below $1$ and yields $s\notin(0,1)$. Also, uniqueness follows from the monotonicity of $v_1$ in $\tau$. Note that $\tau_v$ depends on the data only through $\alpha$, whereas $\tau_\mu$ depends on $R$. Details are given in Appendix~\ref{app:proof-recovery}.

The prescriptions of order $1/M$ discussed in Section~\ref{sec:introduction} fall in case~4. Holding the prior fixed, the \emph{data-emphasized ELBO} of \citet{harveyLearningHyperparametersDataEmphasized2026} multiplies the likelihood by the ratio of the number of parameters, $M(d+1)$, to the sample size, which is the schedule $T_M=\tau/M$ with $\tau=n/(d+1)$. Its limiting mean is then the least-squares prediction, and, since $\tau_v<d+1$ and $v_1$ increases with $\tau$, its limiting variance exceeds $v_{\text{NNGP}}(x)$ as soon as $n\geq(d+1)^2$. 


Table~\ref{tab:phase-transition} summarizes the phase transitions.

\begin{table}[ht]
    \centering
    \caption{Limit of the variational predictive distribution of $BA^\top x$ under $T_M=\tau/M^c$ (Theorem~\ref{thm:main-theorem} and Corollary~\ref{cor:recovery}). LS: least squares.}
    \label{tab:phase-transition}
    \renewcommand{\arraystretch}{1.3}
    \begin{tabular*}{\linewidth}{@{\extracolsep{\fill}}ccc@{}}
        \toprule
        \textbf{Condition} & \textbf{Limit} & \textbf{Behaviour} \\
        \midrule
        $c<1/2$                   & $\mathcal N\bigl(0,v_0(x)\bigr)$           & Prior \\
        $c=1/2$, $\tau\geq\tau_*$ & $\mathcal N\bigl(0,v_0(x)\bigr)$           & Prior \\
        $c=1/2$, $\tau<\tau_*$    & $\mathcal N\bigl(m_{1/2}(x),v_0(x)\bigr)$  & NNGP mean iff $\tau=\tau_\mu$ \\
        $1/2<c<1$                 & $\mathcal N\bigl(m_\infty(x),v_0(x)\bigr)$ & LS mean, prior variance \\
        $c=1$                     & $\mathcal N\bigl(m_\infty(x),v_1(x)\bigr)$ & NNGP variance iff $\tau=\tau_v$ \\
        $c>1$                     & $\delta_{m_\infty(x)}$                     & Point mass \\
        \bottomrule
    \end{tabular*}
\end{table}

\section{Simulation Study}\label{sec:simulations}
We conduct an empirical study of the large-but-finite-width behaviour of the tempered variational posterior under the model introduced in Sections~\ref{sec:lbnn} and \ref{sec:likelihoodtempering}. We generate a synthetic dataset with $n = 40$, $d = 2$, and $Q = I_d$, and known noise variance $\sigma^2 = 20$. The same dataset is used for every $M \in \{2^i : i \text{ even}, 4 \leq i \leq 28\}$, and eight tempering scales
\begin{align*}
    (c,\tau)\in\{&
    (1/4,\tau_\mu), (1/2,\tau_*/0.7), (1/2,\tau_\mu), (3/4,\tau_\mu),\\
    &
    (0.95,\tau_v), (1,\tau_v), (1.05,\tau_v), (5/4,\tau_\mu)
    \}.
\end{align*}
We compute the mean and variance of the latent network output at $x^* = (1, -1)^\top / \sqrt 2$, excluding observation noise. 

We employ an iterated conditional modes (ICM) algorithm \citep{besagStatisticalAnalysisDirty1986, bishopPatternRecognitionMachine2006}, Theorem \ref{thm:orthogexpecAB}, Remark \ref{rem:tempered-one-active} and Lemma \ref{lem:tempered-zero-means} to find the optimum $q^*$, iterating over the fixed-point equations \eqref{posterior mean a_m}-\eqref{posterior cov b_m}. $\tau_\mu$ and $\tau_v$ are as in Corollary \ref{cor:recovery}, so that the limiting variational mean at $c = 1/2$ and variance at $c = 1$ respectively agree with those of the untempered NNGP posterior. The dataset construction details, numerical procedure, and verification of global minimality of the variational posterior are given in Appendix~\ref{app:simulation-details}. 

\subsection{Results}

Figure~\ref{fig:phase-transition} shows the computed posterior mean and variance at $x^*$ for the eight tempering schedules. At the largest width, $M = 2^{28}$, every curve is within $0.01$ of its limit in Table~\ref{tab:phase-transition}, except the variances for $c = 0.95$ and $c = 1.05$, which approach their limits slowly because $1/(MT_M) = M^{c-1}/\tau$ changes slowly when $|c-1| = 0.05$.

\begin{figure}[ht]
    \centering
    \includegraphics[width=\linewidth]{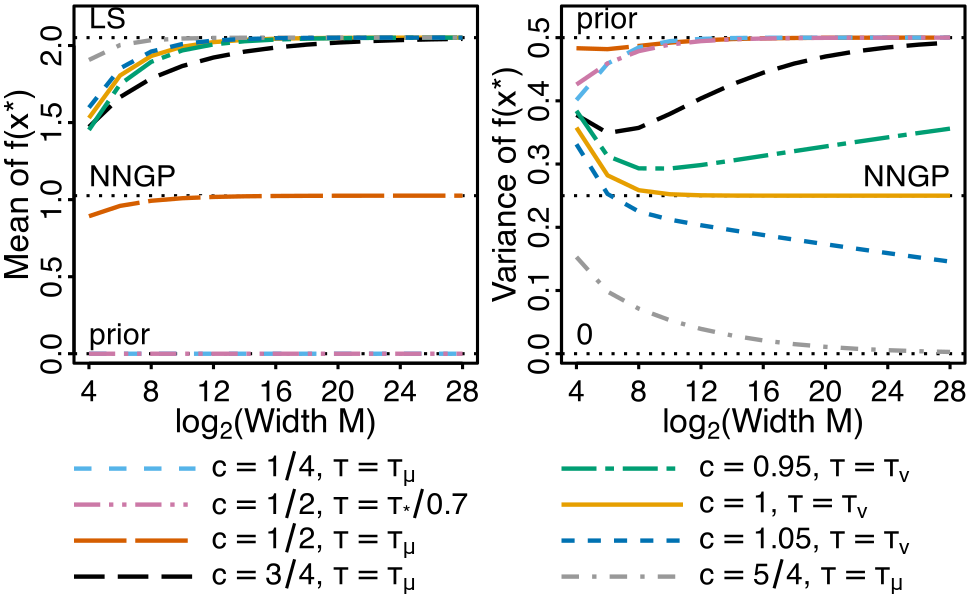}
    \caption{\textit{Exact mean and variance of $f(x^*)$ under the fitted variational distribution for the eight schedules $T_M=\tau/M^c$. Dotted lines indicate the prior parameters, the untempered NNGP posterior ($m_{\text{NNGP}}(x^*)=1.026$, $v_{\text{NNGP}}(x^*)=0.25$), the least-squares prediction $m_\infty(x^*)=2.051$ and zero variance.}}
    \label{fig:phase-transition}
\end{figure}

When $c < 1/2$, the predictive distribution converges to the limiting prior predictive distribution as $M\to\infty$. In particular, the curve for $c = 1/4$ remains close to the prior mean $0$ and the prior variance $1/2$. The same behaviour occurs when $c = 1/2$ and $\tau > \tau_*$: the curve with $c = 1/2$ and $\tau = \tau_*/0.7$ also remains close to the prior reference levels. 

At $c = 1/2$ and $\tau = \tau_\mu$, the posterior mean converges to the untempered NNGP posterior mean, while the posterior variance remains close to the prior variance. For $1/2 < c < 1$, the posterior mean instead approaches the least-squares prediction $Rx^*$. The $c = 3/4$ curve also approaches the prior variance, as expected. The $c = 0.95$ mean also approaches $Rx^*$, but its variance first falls to $0.29$ around $M = 2^{10}$ and then rises up to $0.356$ by $M = 2^{28}$.

At $c = 1$, the untempered NNGP posterior variance is recovered, and the posterior mean approaches the least-squares prediction as well. For $c > 1$, the posterior mean again approaches the least-squares prediction, but the posterior variance collapses towards zero. This is visible in the curves for both $c = 1.05$ and $c = 5/4$, with the variance collapsing more quickly for the larger value of $c$.

\section{Discussion}\label{sec:discussion}
We show that for wide shallow linear BNNs, likelihood tempering provides a partial remedy for prior dominance, in which the variational predictive distribution ignores the data. Under our assumptions, the predictive mean and variance can be recovered at different scales of temperature, $M^{-1/2}$ for mean and $M^{-1}$ for variance. No choice of $\tau,c>0$ in $T_M=\tau/M^c$ recovers both untempered NNGP posterior predictive moments simultaneously. A similar trade-off between recovering the predictive mean and variance under likelihood tempering is discussed by \citet{odgersGaussianMeanField2026} for conjugate linear regression. It remains to study variational approximations of cold posteriors, where the likelihood and prior are jointly rescaled, to determine at which temperature scaling the different predictive moments can potentially be recovered.

While we focus on the predictive distribution at a single input for simplicity, the multivariate central limit theorem would extend our results to finite marginals of this predictive process.

A natural extension would be to investigate if similar results in mean and variance recovery under different tempering scales would be present in nonlinear or deeper networks. Another direction to consider would be to investigate a larger variational family that would allow dependence between each neuron's input and output weights, and not necessarily impose the mean-field assumption. Finally, our results concern global minimizers of the variational objective, and whether gradient-based optimisation reaches them at large width, and whether it exhibits the same phase transitions as in Theorem \ref{thm:main-theorem}, is a separate question.

\bibliographystyle{apalike}
\bibliography{references.bib}

@article{bleiVariationalInferenceReview2017,
  title = {Variational {{Inference}}: {{A Review}} for {{Statisticians}}},
  shorttitle = {Variational {{Inference}}},
  author = {Blei, David M. and Kucukelbir, Alp and McAuliffe, Jon D.},
  year = 2017,
  month = apr,
  journal = {Journal of the American Statistical Association},
  volume = {112},
  number = {518},
  eprint = {1601.00670},
  primaryclass = {stat.CO},
  pages = {859--877},
  issn = {0162-1459, 1537-274X},
  doi = {10.1080/01621459.2017.1285773},
  urldate = {2026-01-30},
  archiveprefix = {arXiv},
  url = {https://doi.org/10.1080/01621459.2017.1285773},
}

@inproceedings{cokerWideMeanFieldBayesian2022a,
  title = {Wide {{Mean-Field Bayesian Neural Networks Ignore}} the {{Data}}},
  booktitle = {Proceedings of {{The}} 25th {{International Conference}} on {{Artificial Intelligence}} and {{Statistics}}},
  author = {Coker, Beau and Bruinsma, Wessel P. and Burt, David R. and Pan, Weiwei and {Doshi-Velez}, Finale},
  year = 2022,
  month = mar,
  pages = {5276--5333},
  publisher = {PMLR},
  issn = {2640-3498},
  urldate = {2026-02-18},
  langid = {english},
  volume = {151},
  series = {Proceedings of Machine Learning Research},
  url = {https://proceedings.mlr.press/v151/coker22a.html},
}

@misc{hronExactPosteriorDistributions2020,
  title = {Exact Posterior Distributions of Wide {{Bayesian}} Neural Networks},
  author = {Hron, Jiri and Bahri, Yasaman and Novak, Roman and Pennington, Jeffrey and {Sohl-Dickstein}, Jascha},
  year = 2020,
  month = nov,
  eprint = {2006.10541},
  primaryclass = {stat.ML},
  doi = {10.48550/arXiv.2006.10541},
  urldate = {2026-02-18},
  archiveprefix = {arXiv},
  url = {https://arxiv.org/abs/2006.10541},
  howpublished = {arXiv preprint},
  note = {Accepted at the {ICML} 2020 Workshop on Uncertainty and Robustness in Deep Learning (non-archival)},
}

@inproceedings{jacotNeuralTangentKernel2018,
 author = {Jacot, Arthur and Gabriel, Franck and Hongler, Clement},
 booktitle = {Advances in Neural Information Processing Systems},
 editor = {S. Bengio and H. Wallach and H. Larochelle and K. Grauman and N. Cesa-Bianchi and R. Garnett},
 pages = {},
 publisher = {Curran Associates, Inc.},
 title = {Neural Tangent Kernel: Convergence and Generalization in Neural Networks},
 url = {https://proceedings.neurips.cc/paper_files/paper/2018/file/5a4be1fa34e62bb8a6ec6b91d2462f5a-Paper.pdf},
 volume = {31},
 year = {2018}
}

@article{nakajimaTheoreticalAnalysisBayesian2011,
  title = {Theoretical {{Analysis}} of {{Bayesian Matrix Factorization}}},
  author = {Nakajima, Shinichi and Sugiyama, Masashi},
  year = 2011,
  journal = {Journal of Machine Learning Research},
  volume = {12},
  number = {79},
  pages = {2583--2648},
  issn = {1533-7928},
  urldate = {2026-02-04},
  month = sep,
  url = {https://www.jmlr.org/papers/v12/nakajima11a.html},
}

@article{nakajimaVariationalBayesSolution2007,
  title = {Variational {{Bayes}} Solution of Linear Neural Networks and Its Generalization Performance},
  author = {Nakajima, Shinichi and Watanabe, Sumio},
  year = 2007,
  month = apr,
  journal = {Neural Computation},
  volume = {19},
  number = {4},
  pages = {1112--1153},
  issn = {0899-7667},
  doi = {10.1162/neco.2007.19.4.1112},
  langid = {english},
  pmid = {17348777},
  url = {https://doi.org/10.1162/neco.2007.19.4.1112},
}

@article{jospinHandsOnBayesianNeural2022,
  title = {Hands-{On} {Bayesian} {Neural} {Networks}—{A} {Tutorial} for {Deep} {Learning} {Users}},
  volume = {17},
  copyright = {https://ieeexplore.ieee.org/Xplorehelp/downloads/license-information/IEEE.html},
  issn = {1556-603X, 1556-6048},
  url = {https://ieeexplore.ieee.org/document/9756596/},
  doi = {10.1109/MCI.2022.3155327},
  language = {en},
  number = {2},
  urldate = {2026-09-14},
  journal = {IEEE Computational Intelligence Magazine},
  author = {Jospin, Laurent Valentin and Laga, Hamid and Boussaid, Farid and Buntine, Wray and Bennamoun, Mohammed},
  month = may,
  year = {2022},
  pages = {29--48},
}

@inproceedings{kendallWhatUncertaintiesWe2017,
  title = {What {Uncertainties} {Do} {We} {Need} in {Bayesian} {Deep} {Learning} for {Computer} {Vision}?},
  volume = {30},
  url = {https://proceedings.neurips.cc/paper_files/paper/2017/hash/2650d6089a6d640c5e85b2b88265dc2b-Abstract.html},
  urldate = {2026-09-14},
  booktitle = {Advances in {Neural} {Information} {Processing} {Systems}},
  publisher = {Curran Associates, Inc.},
  author = {Kendall, Alex and Gal, Yarin},
  year = {2017},
  pages = {5574--5584},
}

@incollection{nealPriorsInfiniteNetworks1996,
  address = {New York, NY},
  title = {Priors for {Infinite} {Networks}},
  isbn = {978-1-4612-0745-0},
  url = {https://doi.org/10.1007/978-1-4612-0745-0_2},
  doi = {10.1007/978-1-4612-0745-0_2},
  language = {en},
  urldate = {2026-09-15},
  booktitle = {Bayesian {Learning} for {Neural} {Networks}},
  publisher = {Springer},
  author = {Neal, Radford M.},
  year = {1996},
  pages = {29--53},
  series = {Lecture Notes in Statistics},
  volume = {118},
}

@inproceedings{matthews2018gaussianprocessbehaviourwide,
  title     = {{Gaussian {P}rocess {B}ehaviour in {W}ide {D}eep {N}eural {N}etworks}},
  author    = {de G. Matthews, Alexander G. and Hron, Jiri and Rowland, Mark and Turner, Richard E. and Ghahramani, Zoubin},
  booktitle = {International Conference on Learning Representations},
  year      = {2018},
  url       = {https://mlanthology.org/iclr/2018/degmatthews2018iclr-gaussian/}
}

@misc{trippe2018overpruningvariationalbayesianneural,
  title = {Overpruning in variational bayesian neural networks},
  url = {https://arxiv.org/abs/1801.06230},
  author = {Trippe, Brian and Turner, Richard},
  year = {2018},
  note = {Presented at the Advances in Approximate Bayesian Inference workshop at {NIPS} 2017},
  eprint = {1801.06230},
  archiveprefix = {arXiv},
  primaryclass = {stat.ML},
  doi = {10.48550/arXiv.1801.06230},
  howpublished = {arXiv preprint},
}

@inproceedings{foongExpressivenessApproximateInference2020,
  title = {On the {Expressiveness} of {Approximate} {Inference} in {Bayesian} {Neural} {Networks}},
  volume = {33},
  url = {https://proceedings.neurips.cc/paper/2020/hash/b6dfd41875bc090bd31d0b1740eb5b1b-Abstract.html},
  urldate = {2026-09-15},
  booktitle = {Advances in {Neural} {Information} {Processing} {Systems}},
  publisher = {Curran Associates, Inc.},
  author = {Foong, Andrew and Burt, David and Li, Yingzhen and Turner, Richard},
  year = {2020},
  pages = {15897--15908},
}

@inproceedings{leeDeepNeuralNetworks2018,
  title = {Deep {Neural} {Networks} as {Gaussian} {Processes}},
  url = {https://openreview.net/forum?id=B1EA-M-0Z},
  language = {en},
  urldate = {2026-09-15},
  author = {Lee, Jaehoon and Bahri, Yasaman and Novak, Roman and Schoenholz, Samuel S. and Pennington, Jeffrey and Sohl-Dickstein, Jascha},
  year = {2018},
  booktitle = {International Conference on Learning Representations},
}

@inproceedings{sunFUNCTIONALVARIATIONALBAYESIAN2018,
  title     = {{Functional Variational Bayesian Neural Networks}},
  author    = {Sun, Shengyang and Zhang, Guodong and Shi, Jiaxin and Grosse, Roger},
  booktitle = {International Conference on Learning Representations},
  year      = {2019},
  url       = {https://mlanthology.org/iclr/2019/sun2019iclr-functional/}
}

@inproceedings{rudnerTractableFunctionSpaceVariational2022,
  title = {Tractable {Function}-{Space} {Variational} {Inference} in {Bayesian} {Neural} {Networks}},
  url = {https://proceedings.neurips.cc/paper_files/paper/2022/hash/8ea50bf458f6070548b11babbe0bf89b-Abstract-Conference.html},
  language = {en},
  urldate = {2026-09-16},
  author = {Rudner, Tim G. J. and Chen, Zonghao and Teh, Yee Whye and Gal, Yarin},
  year = {2022},
  booktitle = {Advances in Neural Information Processing Systems},
  volume = {35},
  pages = {22686--22698},
  publisher = {Curran Associates, Inc.},
  doi = {10.52202/068431-1648},
}

@inproceedings{zhangNoisyNaturalGradient2018,
  title = {Noisy {Natural} {Gradient} as {Variational} {Inference}},
  issn = {2640-3498},
  url = {https://proceedings.mlr.press/v80/zhang18l.html},
  language = {en},
  urldate = {2026-09-16},
  booktitle = {Proceedings of the 35th {International} {Conference} on {Machine} {Learning}},
  publisher = {PMLR},
  author = {Zhang, Guodong and Sun, Shengyang and Duvenaud, David and Grosse, Roger},
  month = jul,
  year = {2018},
  pages = {5852--5861},
  volume = {80},
  series = {Proceedings of Machine Learning Research},
}

@inproceedings{osawaPracticalDeepLearning2019,
  title = {Practical {Deep} {Learning} with {Bayesian} {Principles}},
  volume = {32},
  url = {https://proceedings.neurips.cc/paper_files/paper/2019/hash/b53477c2821c1bf0da5d40e57b870d35-Abstract.html},
  urldate = {2026-09-14},
  booktitle = {Advances in {Neural} {Information} {Processing} {Systems}},
  publisher = {Curran Associates, Inc.},
  author = {Osawa, Kazuki and Swaroop, Siddharth and Khan, Mohammad Emtiyaz and Jain, Anirudh and Eschenhagen, Runa and Turner, Richard and Yokota, Rio},
  year = {2019},
  pages = {4287--4299},
}

@inproceedings{wenzelHowGoodBayes2020,
  title = {How {Good} is the {Bayes} {Posterior} in {Deep} {Neural} {Networks} {Really}?},
  issn = {2640-3498},
  url = {https://proceedings.mlr.press/v119/wenzel20a.html},
  language = {en},
  urldate = {2026-09-16},
  booktitle = {Proceedings of the 37th {International} {Conference} on {Machine} {Learning}},
  publisher = {PMLR},
  author = {Wenzel, Florian and Roth, Kevin and Veeling, Bastiaan and Swiatkowski, Jakub and Tran, Linh and Mandt, Stephan and Snoek, Jasper and Salimans, Tim and Jenatton, Rodolphe and Nowozin, Sebastian},
  month = jul,
  year = {2020},
  pages = {10248--10259},
  volume = {119},
  series = {Proceedings of Machine Learning Research},
}

@misc{huixVariationalInferenceOverparameterized2022,
  title = {Variational {Inference} of overparameterized {Bayesian} {Neural} {Networks}: a theoretical and empirical study},
  shorttitle = {Variational {Inference} of overparameterized {Bayesian} {Neural} {Networks}},
  url = {https://arxiv.org/abs/2207.03859},
  doi = {10.48550/arXiv.2207.03859},
  urldate = {2026-09-14},
  author = {Huix, Tom and Majewski, Szymon and Durmus, Alain and Moulines, Eric and Korba, Anna},
  month = jul,
  year = {2022},
  eprint = {2207.03859},
  archiveprefix = {arXiv},
  primaryclass = {stat.ML},
  howpublished = {arXiv preprint},
}

@inproceedings{harveyLearningHyperparametersDataEmphasized2026,
  title = {Learning {Hyperparameters} via a {Data}-{Emphasized} {Variational} {Objective}},
  booktitle = {Proceedings of the 29th {International} {Conference} on {Artificial} {Intelligence} and {Statistics}},
  author = {Harvey, Ethan and Petrov, Mikhail and Hughes, Michael C.},
  year = {2026},
  volume = {300},
  series = {Proceedings of Machine Learning Research},
  pages = {2143--2151},
  publisher = {PMLR},
  month = may,
  url = {https://proceedings.mlr.press/v300/harvey26a.html},
  eprint = {2502.01861},
  archiveprefix = {arXiv},
  primaryclass = {cs.LG},
}

@inproceedings{odgersGaussianMeanField2026,
  title = {Gaussian {Mean} {Field} {Variational} {Inference} can {Overestimate} {Predictive} {Variance}},
  url = {https://openreview.net/forum?id=RG7maF4bGu},
  language = {en},
  urldate = {2026-09-16},
  booktitle = {Proceedings of the 43rd International Conference on Machine Learning},
  author = {Odgers, James and Riegler, Ben and Swaroop, Siddharth and Fortuin, Vincent},
  month = jul,
  year = {2026},
  volume = {306},
  series = {Proceedings of Machine Learning Research},
  publisher = {PMLR},
  eprint = {2606.25745},
  archiveprefix = {arXiv},
  primaryclass = {stat.ML},
}

@article{bissiriGeneralFrameworkUpdating2016a,
    title = {A general framework for updating belief distributions},
    volume = {78},
    issn = {1467-9868},
    url = {https://onlinelibrary.wiley.com/doi/abs/10.1111/rssb.12158},
    doi = {10.1111/rssb.12158},
    language = {en},
    number = {5},
    urldate = {2026-08-31},
    journal = {Journal of the Royal Statistical Society: Series B (Statistical Methodology)},
    author = {Bissiri, P. G. and Holmes, C. C. and Walker, S. G.},
    year = {2016},
    pages = {1103--1130},
}

@article{yang$alpha$variationalInference2020,
    title = {\${\textbackslash}alpha \$-variational inference with statistical guarantees},
    volume = {48},
    issn = {0090-5364, 2168-8966},
    url = {https://projecteuclid.org/journals/annals-of-statistics/volume-48/issue-2/alpha--variational-inference-with-statistical-guarantees/10.1214/19-AOS1827.full},
    doi = {10.1214/19-AOS1827},
    language = {en},
    number = {2},
    urldate = {2026-09-29},
    journal = {The Annals of Statistics},
    publisher = {Institute of Mathematical Statistics},
    author = {Yang, Yun and Pati, Debdeep and Bhattacharya, Anirban},
    month = apr,
    year = {2020},
    pages = {886--905},
}

@article{MacKay1992,
author = {MacKay,David J. C.},
title = {A Practical
Bayesian Framework for Backpropagation Networks},
journal = {Neural Computation},
volume = {4},
number = {3},
pages = {448-472},
year = {1992},
month = {05},
issn = {0899-7667},
doi = {10.1162/neco.1992.4.3.448},
url =
{https://doi.org/10.1162/neco.1992.4.3.448},
eprint = {https://direct.mit.edu/neco/article-pdf/4/3/448/812348/neco.1992.4.3.448.pdf}
}

@article{Ghahramani2015,
	author = {Ghahramani, Zoubin},
	date = {2015/05/01},
	doi = {10.1038/nature14541},
	id = {Ghahramani2015},
	isbn = {1476-4687},
	journal = {Nature},
	number = {7553},
	pages = {452--459},
	title = {Probabilistic machine learning and artificial intelligence},
	url = {https://doi.org/10.1038/nature14541},
	volume = {521},
	year = {2015}}

@article{Arbel2023,
author = {Julyan Arbel and Konstantinos Pitas and Mariia Vladimirova and Vincent Fortuin},
title = {{A Primer on Bayesian Neural Networks: Review and Debates}},
volume = {41},
journal = {Statistical Science},
number = {2},
publisher = {Institute of Mathematical Statistics},
pages = {316 -- 353},
year = {2026},
doi = {10.1214/24-STS969},
URL = {https://doi.org/10.1214/24-STS969}
}

@InProceedings{Izmailov2021,
  title = 	 {What Are Bayesian Neural Network Posteriors Really Like?},
  author =       {Izmailov, Pavel and Vikram, Sharad and Hoffman, Matthew D and Wilson, Andrew Gordon Gordon},
  booktitle = 	 {Proceedings of the 38th International Conference on Machine Learning},
  pages = 	 {4629--4640},
  year = 	 {2021},
  editor = 	 {Meila, Marina and Zhang, Tong},
  volume = 	 {139},
  series = 	 {Proceedings of Machine Learning Research},
  month = 	 {18--24 Jul},
  publisher =    {PMLR},
  url = 	 {https://proceedings.mlr.press/v139/izmailov21a.html}
}

@article{Hoffman2013,
  author  = {Matthew D. Hoffman and David M. Blei and Chong Wang and John Paisley},
  title   = {Stochastic Variational Inference},
  journal = {Journal of Machine Learning Research},
  year    = {2013},
  volume  = {14},
  number  = {4},
  pages   = {1303-1347},
  url     = {http://jmlr.org/papers/v14/hoffman13a.html}
}

@inproceedings{KingmaWelling2014,
  title     = {{Auto-Encoding Variational Bayes}},
  author    = {Kingma, Diederik P. and Welling, Max},
  booktitle = {International Conference on Learning Representations},
  year      = {2014},
  url       = {https://mlanthology.org/iclr/2014/kingma2014iclr-auto/}
}

@inproceedings{Graves2011,
 author = {Graves, Alex},
 booktitle = {Advances in Neural Information Processing Systems},
 editor = {J. Shawe-Taylor and R. Zemel and P. Bartlett and F. Pereira and K. Weinberger},
 pages = {},
 publisher = {Curran Associates, Inc.},
 title = {Practical Variational Inference for Neural Networks},
 url = {https://proceedings.neurips.cc/paper_files/paper/2011/file/7eb3c8be3d411e8ebfab08eba5f49632-Paper.pdf},
 volume = {24},
 year = {2011}
}

@InProceedings{Blundell2015,
  title = 	 {Weight Uncertainty in Neural Network},
  author = 	 {Blundell, Charles and Cornebise, Julien and Kavukcuoglu, Koray and Wierstra, Daan},
  booktitle = 	 {Proceedings of the 32nd International Conference on Machine Learning},
  pages = 	 {1613--1622},
  year = 	 {2015},
  editor = 	 {Bach, Francis and Blei, David},
  volume = 	 {37},
  series = 	 {Proceedings of Machine Learning Research},
  address = 	 {Lille, France},
  month = 	 {07--09 Jul},
  publisher =    {PMLR},
  url = 	 {https://proceedings.mlr.press/v37/blundell15.html}
}

@InProceedings{Shen2024,
  title = 	 {Variational Learning is Effective for Large Deep Networks},
  author =       {Shen, Yuesong and Daheim, Nico and Cong, Bai and Nickl, Peter and Marconi, Gian Maria and Raoul, Bazan Clement Emile Marcel and Yokota, Rio and Gurevych, Iryna and Cremers, Daniel and Khan, Mohammad Emtiyaz and M\"{o}llenhoff, Thomas},
  booktitle = 	 {Proceedings of the 41st International Conference on Machine Learning},
  pages = 	 {44665--44686},
  year = 	 {2024},
  editor = 	 {Salakhutdinov, Ruslan and Kolter, Zico and Heller, Katherine and Weller, Adrian and Oliver, Nuria and Scarlett, Jonathan and Berkenkamp, Felix},
  volume = 	 {235},
  series = 	 {Proceedings of Machine Learning Research},
  month = 	 {21--27 Jul},
  publisher =    {PMLR},
  url = 	 {https://proceedings.mlr.press/v235/shen24b.html}
}

@techreport{Minka2005,
author = {Minka, Tom},
title = {Divergence Measures and Message Passing},
year = {2005},
month = {January},
url = {https://www.microsoft.com/en-us/research/publication/divergence-measures-and-message-passing/},
pages = {17},
institution = {Microsoft Research},
number = {MSR-TR-2005-173},
}

@InProceedings{MargossianSaul2023,
  title = 	 {The Shrinkage-Delinkage Trade-off: an Analysis of Factorized {G}aussian Approximations for Variational Inference},
  author =       {Margossian, Charles C. and Saul, Lawrence K.},
  booktitle = 	 {Proceedings of the Thirty-Ninth Conference on Uncertainty in Artificial Intelligence},
  pages = 	 {1358--1367},
  year = 	 {2023},
  editor = 	 {Evans, Robin J. and Shpitser, Ilya},
  volume = 	 {216},
  series = 	 {Proceedings of Machine Learning Research},
  month = 	 {31 Jul--04 Aug},
  publisher =    {PMLR},
  url = 	 {https://proceedings.mlr.press/v216/margossian23a.html}
}

@InProceedings{Descours2023,
  title = 	 {Law of Large Numbers for Bayesian two-layer Neural Network trained with Variational Inference},
  author =       {Descours, Arnaud and Huix, Tom and Guillin, Arnaud and Michel, Manon and Moulines, {\'E}ric and Nectoux, Boris},
  booktitle = 	 {Proceedings of Thirty Sixth Conference on Learning Theory},
  pages = 	 {4657--4695},
  year = 	 {2023},
  editor = 	 {Neu, Gergely and Rosasco, Lorenzo},
  volume = 	 {195},
  series = 	 {Proceedings of Machine Learning Research},
  month = 	 {12--15 Jul},
  publisher =    {PMLR},
  url = 	 {https://proceedings.mlr.press/v195/descours23a.html}
}

@inproceedings{Burt2021,
 author    = {David R. Burt and Sebastian W. Ober and Adri{\`a} Garriga-Alonso and Mark van der Wilk},
 booktitle = {Third Symposium on Advances in Approximate Bayesian Inference},
 title     = {Understanding Variational Inference in Function-Space},
 year      = {2021}
}

@InProceedings{CinquinBamler2025,
  title = 	 {Well-Defined Function-Space Variational Inference in Bayesian Neural Networks via Regularized KL-Divergence},
  author =       {Cinquin, Tristan and Bamler, Robert},
  booktitle = 	 {Proceedings of the Forty-first Conference on Uncertainty in Artificial Intelligence},
  pages = 	 {752--776},
  year = 	 {2025},
  editor = 	 {Chiappa, Silvia and Magliacane, Sara},
  volume = 	 {286},
  series = 	 {Proceedings of Machine Learning Research},
  month = 	 {21--25 Jul},
  publisher =    {PMLR},
  url = 	 {https://proceedings.mlr.press/v286/cinquin25a.html}
}

@inproceedings{Aitchison2021,
  title     = {{A Statistical Theory of Cold Posteriors in Deep Neural Networks}},
  author    = {Aitchison, Laurence},
  booktitle = {International Conference on Learning Representations},
  year      = {2021},
  url       = {https://mlanthology.org/iclr/2021/aitchison2021iclr-statistical/}
}

@article{
Mei2018,
author = {Song Mei  and Andrea Montanari  and Phan-Minh Nguyen },
title = {A mean field view of the landscape of two-layer neural networks},
journal = {Proceedings of the National Academy of Sciences},
volume = {115},
number = {33},
pages = {E7665-E7671},
year = {2018},
doi = {10.1073/pnas.1806579115},
URL = {https://www.pnas.org/doi/abs/10.1073/pnas.1806579115},
eprint = {https://www.pnas.org/doi/pdf/10.1073/pnas.1806579115}}

@inproceedings{ChizatBach2018,
  title     = {{On the Global Convergence of Gradient Descent for Over-Parameterized Models Using Optimal Transport}},
  author    = {Chizat, Lénaïc and Bach, Francis},
  booktitle = {Neural Information Processing Systems},
  year      = {2018},
  pages     = {3036-3046},
  url       = {https://mlanthology.org/neurips/2018/chizat2018neurips-global/}
}

@article{RotskoffVandenEijnden2022,
author = {Rotskoff, Grant and Vanden-Eijnden, Eric},
title = {Trainability and Accuracy of Artificial Neural Networks: An Interacting Particle System Approach},
journal = {Communications on Pure and Applied Mathematics},
volume = {75},
number = {9},
pages = {1889-1935},
doi = {https://doi.org/10.1002/cpa.22074},
url = {https://onlinelibrary.wiley.com/doi/abs/10.1002/cpa.22074},
eprint = {https://onlinelibrary.wiley.com/doi/pdf/10.1002/cpa.22074},
year = {2022}
}

@article{SirignanoSpiliopoulos2020,
author = {Sirignano, Justin and Spiliopoulos, Konstantinos},
title = {Mean Field Analysis of Neural Networks: A Law of Large Numbers},
journal = {SIAM Journal on Applied Mathematics},
volume = {80},
number = {2},
pages = {725-752},
year = {2020},
doi = {10.1137/18M1192184},

URL = { 
    
        https://doi.org/10.1137/18M1192184
    
    

},
eprint = { 
    
        https://doi.org/10.1137/18M1192184
    
    

}
}

@inproceedings{Saxe2014,
  title     = {{Exact Solutions to the Nonlinear Dynamics of Learning in Deep Linear Neural Networks}},
  author    = {Saxe, Andrew M. and McClelland, James L. and Ganguli, Surya},
  booktitle = {International Conference on Learning Representations},
  year      = {2014},
  url       = {https://mlanthology.org/iclr/2014/saxe2014iclr-exact/}
}

@article{nakajimaGlobalAnalyticSolution2013,
  title     = {{Global Analytic Solution of Fully-Observed Variational Bayesian Matrix Factorization}},
  author    = {Nakajima, Shinichi and Sugiyama, Masashi and Babacan, S. Derin and Tomioka, Ryota},
  journal   = {Journal of Machine Learning Research},
  year      = {2013},
  pages     = {1-37},
  volume    = {14},
  url       = {https://mlanthology.org/jmlr/2013/nakajima2013jmlr-global/}
}

@article{besagStatisticalAnalysisDirty1986,
    title = {On the {Statistical} {Analysis} of {Dirty} {Pictures}},
    volume = {48},
    issn = {0035-9246},
    url = {https://www.jstor.org/stable/2345426},
    number = {3},
    urldate = {2026-09-30},
    journal = {Journal of the Royal Statistical Society. Series B (Methodological)},
    publisher = {[Royal Statistical Society, Oxford University Press]},
    author = {Besag, Julian},
    year = {1986},
    pages = {259--302},
}

@book{bishopPatternRecognitionMachine2006,
    address = {New York},
    series = {Information science and statistics},
    title = {{Pattern recognition and machine learning}},
    isbn = {978-0-387-31073-2},
    language = {en},
    publisher = {Springer},
    author = {Bishop, Christopher M.},
    year = {2006},
}

@book{rasmussenGaussianProcessesMachine2006,
     title     = {Gaussian Processes for Machine Learning},
     author    = {Rasmussen, Carl Edward and Williams, Christopher K. I.},
     publisher = {MIT Press},
     address   = {Cambridge, MA},
     year      = {2006}
   }

\appendix
\thispagestyle{empty}

\section{Supplementary Materials for Section~\texorpdfstring{\refLBNN}{Section}}\label{app:supplementary-LBNN}
\begin{remark}\label{rem:whitening}
    While seemingly a limiting condition, note that for any design matrix $X \in \R^{n\times d}$ with full column rank, the isotropic second moment condition can be enforced via a linear transformation on $X$. Write the singular value decomposition (SVD) of $X$ as $X^\top = UDV^\top$ where $U$ and $V$ each have orthonormal columns and $D$ contains the positive singular values of $X$. Define $\xi = \sqrt n D^{-1} U^\top$ and $\widetilde X^\top = \xi X^\top$.  The transformed covariate vectors $\widetilde x_i = \xi x_i$ satisfy $\tilde Q = \xi Q\xi^\top = I_d$. 
    Variational inference may be performed using $\widetilde X$ and the stated isotropic prior on the transformed input weights $\widetilde A$. The fitted variational posterior distribution can then be mapped back to the original covariates through $A = \xi^\top \widetilde A$, since $ B\widetilde A^\top \widetilde x = BA^\top x$ and $\widetilde x = \xi x$.  The corresponding prior on the original weights is $a_m\sim\mathcal N(0,Q^{-1}/d)$, which generally differs from the isotropic prior in \eqref{eq:prior-A}. 
    Since the whitening transformation is independent of $M$, the untempered variational predictive distribution on the original design still converges to the corresponding prior predictive distribution under the induced prior.
\end{remark}

In this appendix, we assume $A \in \R^{d\times M}$ and $B \in \R^{1\times M}$ with priors as defined in \eqref{eq:prior-A} and \eqref{eq:prior-B}, where $p(A, B) = p(A)p(B)$ is the density of the prior. 
Since $BA^\top = W \in \R^{1\times d}$ is a row vector, one neuron suffices to represent it: take $a_1 = W^\top$, $b_1 = 1$, and $a_m = b_m = 0$ for $m \geq 2$. Then $BA^\top = W$, the columns of $A$ are pairwise orthogonal, and at most one entry of $B$ is nonzero.

\begin{theorem}\label{thm:orthogAB}
    Let $A$ and $B$ be as defined in Section~\ref{sec:lbnn} and assume prior distributions defined in \eqref{eq:prior-A} and \eqref{eq:prior-B} for $A$ and $B$ respectively. Fix $W \in \R^{1\times d}$. Among all pairs $(A, B)$ satisfying $BA^\top = W$, there exists a minimizer of $d\|A\|_F^2 + M\|B\|_2^2$ such that the columns of $A$ are pairwise orthogonal and at most one entry of $B$ is nonzero. The minimum then equals $2\sqrt{d M}\|W\|_2$.
\end{theorem}
\begin{proof}
    We wish to solve the minimization problem 
    \begin{equation}\label{eq:minproblem}
        \min_{W = BA^\top} \left[d\|A\|_F^2 + M\|B\|_2^2\right]
    \end{equation}
    By the AM-GM and Cauchy-Schwarz inequalities,
    \begin{align*}
        d\|A\|_F^2 + M\|B\|_2^2 &\geq 2\sqrt{dM}\|A\|_F\|B\|_2\\
        &\geq 2\sqrt{dM}\|BA^\top\|_2 = 2\sqrt{dM}\|W\|_2.
    \end{align*}
    If $W = 0$, this bound is attained by taking $A = B = 0$. For $W \neq 0$, take
    \begin{equation*}
        a_1 = \left(\frac{M}{d}\right)^{1/4}\frac{W^\top}{\|W\|_2^{1/2}}, \qquad b_1 = \left(\frac{d}{M}\right)^{1/4}\|W\|_2^{1/2},
    \end{equation*}
    and set $a_m = b_m = 0$ for $m \geq 2$. Then $BA^\top = W$ and
    \begin{equation*}
        d\|A\|_F^2 + M\|B\|_2^2 = 2\sqrt{dM}\|W\|_2.
    \end{equation*}
    Thus, this choice of $A$ and $B$ solves \eqref{eq:minproblem}. The columns of $A$ are pairwise orthogonal and at most one entry of $B$ is nonzero, which proves the claim.
\end{proof}

\subsection{Proof of Proposition~\ref{thm:distBAxprior}}\label{app:proofdistBAxprior}
\begin{proof}
    Let $g_m \coloneq a_m^\top x$. Then
    \begin{equation*}
    g_m \iidsim \mathcal N\left(0,\frac {x^\top x}d\right)
    \end{equation*}
    under the prior, and hence $g \coloneq A^\top x = (g_1,\dots,g_M)^\top \sim \mathcal N\left(0,\frac{x^\top x}{d}I_M\right)$.
    Conditioning on $g$, since $B^\top\coloneq(b_1,\dots,b_M)^\top \sim \mathcal N_M(0,\frac{1}{M}I_M)$, $BA^\top x \mid A^\top x=g$ is almost surely equal to
    \begin{equation*}
     B g \sim \mathcal N\left(0,\frac{1}{M}\|g\|_2^2\right).
    \end{equation*}
    Moreover, by the law of large numbers,
    \begin{equation*}
    \frac{1}{M}\|g\|_2^2 = \frac{1}{M}\sum_{m=1}^M g_m^2 \inp  \mathbb E(g_1^2) = \frac{x^\top x}{d}.
    \end{equation*}
    Therefore, $BA^\top x$ has the same distribution as $\sqrt{\frac{1}{M}\|g\|_2^2 }\, G$ for $G\sim \mathcal N (0,1)$ independent of $g$, and by Slutsky's theorem,
    \begin{equation*}
    BA^\top x \ind \mathcal N\left(0,\frac {x^\top x}d\right).
    \end{equation*}
\end{proof}

\subsection{Proof of Theorem~\ref{thm:orthogexpecAB}}\label{app:orthogexpecAB}

\begin{proof}
    For each fixed $M$, we first show the existence of a minimizer of the objective in \eqref{eq:minproblemexpanded}. Since $\Sigma_{a_m} \succ 0$, 
    \begin{equation*}
        x_i^\top \Sigma_{a_m} x_i = \|\Sigma_{a_m}^{1/2} x_i \|_2^2 \geq 0.
    \end{equation*}
    Since $\Sigma_{b_m} > 0$ and $\sigma^2 > 0$, then for the terms in \eqref{eq:minproblemexpanded} that depend on $x_i$, 
    \begin{equation*}
        \frac{1}{\sigma^2}\sum_{i=1}^n \left[\left(y_i - \sum_{m=1}^M \mu_{b_m}\mu_{a_m}^\top x_i \right)^2 + \sum_{m=1}^M \left((\mu_{b_m}^2 + \Sigma_{b_m}) \|\Sigma_{a_m}^{1/2} x_i\|_2^2 + \Sigma_{b_m}(\mu_{a_m}^\top x_i)^2\right)\right] \geq 0
    \end{equation*}
    due to every summand being nonnegative. For the sum of the remaining terms, 
    \begin{equation}\label{eq:unbounded-means}
        \sum_{m=1}^M [d\|\mu_{a_m}\|_2^2 + M\mu_{b_m}^2] \geq \sum_{m=1}^M [\|\mu_{a_m}\|_2^2 + \mu_{b_m}^2] 
    \end{equation}
    If any $\mu_{b_m}$ or $\|\mu_{a_m}\|_2$ becomes unbounded, then the left term in \eqref{eq:unbounded-means} diverges. Let $$\lambda_{m,1},\ldots,\lambda_{m,d}>0$$ be the eigenvalues of $\Sigma_{a_m}$ and define $h(u)=u-\log u-1$ for $u>0$. Then
    \begin{equation}\label{eq:covariance-penalty-limits}
        h(u)\geq0,\qquad \lim_{u\downarrow0}h(u)=\infty,\qquad \lim_{u\to\infty}h(u)=\infty.
    \end{equation}
    For the terms involving the covariances $\Sigma_{a_m}$ and $\Sigma_{b_m}$,
    \begin{equation}\label{eq:covariance-penalty-identities}
    \begin{aligned}
        d\tr(\Sigma_{a_m})-\log|\Sigma_{a_m}|-d\log d &= d+\sum_{j=1}^d h(d\lambda_{m,j}),\\
        M\Sigma_{b_m}-\log\Sigma_{b_m}-\log M &= 1+h(M\Sigma_{b_m}).
    \end{aligned}
    \end{equation}
    Using \eqref{eq:covariance-penalty-identities} and the fact that the terms in \eqref{eq:minproblemexpanded} that depend on $x_i$ are nonnegative, we have
    \begin{equation}\label{eq:objective-lower-bound}
        \mathcal L^{\text{VB}}(q)-M(d+1+d\log d +\log M)\geq\sum_{m=1}^M\Bigl[d\|\mu_{a_m}\|_2^2+M\mu_{b_m}^2+\sum_{j=1}^d h(d\lambda_{m,j})+h(M\Sigma_{b_m})\Bigr]\geq0.
    \end{equation}
    So, \eqref{eq:objective-lower-bound} and \eqref{eq:covariance-penalty-limits} show that the objective in \eqref{eq:minproblemexpanded} is coercive. They also show that any sequence along which the objective is bounded above has bounded means and covariances, with covariance eigenvalues and scalar variances bounded away from zero. Thus, we can choose a subsequence along which the means and covariances converge, with the limiting covariances still satisfying $\Sigma_{a_m}\succ0$ and $\Sigma_{b_m}>0$.

    Since the objective is continuous, the bound in \eqref{eq:objective-lower-bound} and the limits in \eqref{eq:covariance-penalty-limits} imply that it attains a global minimum. The same argument applies to the tempered objective for each fixed $T>0$, with $\sigma^2$ replaced by $T\sigma^2$.

    Let $q^*$ be a global minimizer with means $\hat\mu_{a_m}$ and $\hat\mu_{b_m}$, and set
    \begin{equation*}
        \widehat W = \sum_{m=1}^M \hat\mu_{b_m}\hat\mu_{a_m}^\top.
    \end{equation*}
    Hold $\Sigma_{a_m}$ and $\Sigma_{b_m}$ fixed at their values under $q^*$ and optimize $\mathcal L^{\text{VB}}(q)$ with regards to the mean $\mu_{b_m},\mu_{a_m}$ only, under the constraint
    \begin{equation*}
        \sum_{m=1}^M \mu_{b_m}\mu_{a_m}^\top = \widehat W.
    \end{equation*}
    Rewrite \eqref{eq:minproblemexpanded} and neglect some constant terms to obtain 
    \begin{equation*}
    \begin{aligned}
        \mathcal L^{\text{VB}}(r) = {}&C(\Sigma_{a_m}, \Sigma_{b_m}) + \frac{1}{\sigma^2}\sum_{i=1}^n \biggl(y_i - \sum_{m=1}^M \mu_{b_m} \mu_{a_m}^\top x_i\biggr)^2 \\
        &+ \sum_{m=1}^M \biggl[\biggl(d+\frac{n}{\sigma^2}\Sigma_{b_m}\biggr)\|\mu_{a_m}\|_2^2 + \mu_{b_m}^2 \biggl(M+\frac{n}{\sigma^2}\tr(\Sigma_{a_m})\biggr)\biggr],
    \end{aligned}
    \end{equation*}
    where $C(\Sigma_{a_m}, \Sigma_{b_m})$ is some constant term. Notice how the objective takes the form of a penalized regression problem, where the squared-error term measures the fit to the data and the final sum over $m$ is the penalty term.  Since $\mu_B\mu_A^\top = \widehat W$, the fit term is constant. Thus, it suffices to minimize only the penalty term while keeping $\mu_B \mu_A^\top  = \widehat W$. The penalty term simplifies 
    \begin{equation*}
        P = \sum_{m=1}^M \left[\left(d+\frac{n}{\sigma^2}\Sigma_{b_m}\right)\|\mu_{a_m}\|_2^2 + \mu_{b_m}^2 \left(M+\frac{n}{\sigma^2}\tr(\Sigma_{a_m})\right)\right] \coloneq \sum_{m=1}^M [\mu_{a_m}^\top \eta_{a_m} \mu_{a_m} + \eta_{b_m} \mu_{b_m}^2],
    \end{equation*}
    where
    \begin{equation*}
        \kappa_m=d+\frac{n}{\sigma^2}\Sigma_{b_m}, \qquad \eta_{a_m}=\kappa_m I_d, \qquad \eta_{b_m}=M+\frac{n}{\sigma^2}\tr(\Sigma_{a_m}).
    \end{equation*}

    Let
    \begin{equation*}
        \zeta = \min_m \sqrt{\kappa_m\eta_{b_m}}.
    \end{equation*}
    By the AM-GM and triangle inequalities,
    \begin{align*}
        P &\geq 2\sum_{m=1}^M \sqrt{\kappa_m\eta_{b_m}}\,|\mu_{b_m}|\,\|\mu_{a_m}\|_2\\
        &\geq 2\zeta\sum_{m=1}^M |\mu_{b_m}|\,\|\mu_{a_m}\|_2\\
        &\geq 2\zeta\left\|\sum_{m=1}^M \mu_{b_m}\mu_{a_m}^\top\right\|_2\\
        &= 2\zeta\|\widehat W\|_2.
    \end{align*}
    Therefore, $P$ is bounded below by $2\zeta\|\widehat W\|_2$.

    To show that this bound is attained, choose an index $j$ such that $\zeta = \sqrt{\kappa_j\eta_{b_j}}$. If $\widehat W \neq 0$, let
    \begin{equation*}
        t = \left(\frac{\eta_{b_j}}{\kappa_j}\right)^{1/4}\|\widehat W\|_2^{1/2},
    \end{equation*}
    and take
    \begin{equation*}
        \mu_{a_j} = t\frac{\widehat W^\top}{\|\widehat W\|_2}, \qquad
        \mu_{b_j} = \frac{\|\widehat W\|_2}{t}.
    \end{equation*}
    Set $\mu_{a_m} = 0$ and $\mu_{b_m} = 0$ for all $m \neq j$. Then,
    \begin{equation*}
        \sum_{m=1}^M \mu_{b_m}\mu_{a_m}^\top = \mu_{b_j}\mu_{a_j}^\top = \widehat W,
    \end{equation*}
    and
    \begin{align*}
        P &= \kappa_j t^2 + \eta_{b_j}\frac{\|\widehat W\|_2^2}{t^2}\\
        &= 2\sqrt{\kappa_j\eta_{b_j}}\,\|\widehat W\|_2\\
        &= 2\zeta\|\widehat W\|_2.
    \end{align*}
    If $\widehat W = 0$, the bound is attained by taking all means to be zero.

    Denote by $q'$ the distribution with these means and the covariances of $q^*$. It has the same data-fit term as $q^*$, and its penalty is at most $P$ by construction. Since $q^*$ is a global minimizer, $q'$ must also be a global minimizer. By construction, at most one neuron is active under $q'$.

    It remains to show that at most one neuron is active under $q^*$ as well. The objective is smooth on the open set of admissible parameters $\mu_{a_m}, \mu_{b_m}, \Sigma_{a_m}, \Sigma_{b_m}$, so every global minimizer is a stationary point. At any stationary point, differentiating \eqref{eq:minproblemexpanded} with respect to $\mu_{a_m}$ gives for  $W = \sum_{m=1}^M \mu_{b_m}\mu_{a_m}^\top$
    \begin{equation*}
        \left(d + \frac{n}{\sigma^2}\Sigma_{b_m}\right)\mu_{a_m}
        = \frac{n}{\sigma^2}\mu_{b_m}(R - W)^\top.
    \end{equation*}
    The scalar coefficient on the left is positive, so $\mu_{b_m} = 0$ implies $\mu_{a_m} = 0$. Thus, it suffices to show that $\hat\mu_{b_m}$ is nonzero for at most one index $m$.

    Differentiating \eqref{eq:minproblemexpanded} with respect to $\Sigma_{a_m}$ gives
    \begin{align*}
        \frac{\partial\mathcal L^{\text{VB}}(q)}{\partial\Sigma_{a_m}}
        &= dI_d - \Sigma_{a_m}^{-1} + \frac{1}{\sigma^2}\left(\mu_{b_m}^2+\Sigma_{b_m}\right)\sum_{i=1}^n x_ix_i^\top\\
        &= \left(d+\frac{n}{\sigma^2}\left[\mu_{b_m}^2+\Sigma_{b_m}\right]\right)I_d - \Sigma_{a_m}^{-1},
    \end{align*}
    where we used $\sum_{i=1}^n x_ix_i^\top=nQ=nI_d$. Setting the derivative equal to zero gives
    \begin{equation*}
        \Sigma_{a_m}^{-1} = \left(d+\frac{n}{\sigma^2}\left[\mu_{b_m}^2+\Sigma_{b_m}\right]\right)I_d.
    \end{equation*}
    Since $d>0$, $\mu_{b_m}^2\geq0$ and $\Sigma_{b_m}>0$, the coefficient of $I_d$ is positive, so inverting both sides gives
    \begin{equation*}
        \Sigma_{a_m} = \frac{I_d}{d+\frac{n}{\sigma^2}\left[\mu_{b_m}^2+\Sigma_{b_m}\right]},
    \end{equation*}
    which is \eqref{posterior cov a_m}. No parameter of any other neuron appears in this derivative, because the only term of \eqref{eq:minproblemexpanded} that couples the neurons, $\|y_i-\sum_{k=1}^M\mu_{b_k}\mu_{a_k}^\top x_i\|_2^2$, does not depend on $\Sigma_{a_m}$. Hence, \eqref{posterior cov a_m} holds for every $m$ at every stationary point, regardless of the number of active neurons.

    Suppose, for contradiction, that there are two neurons at which $\hat\mu_{b_m} \neq 0$. Choose one of these neurons whose index $m$ differs from $j$. Its mean $\mu_{b_m}$ is $\hat\mu_{b_m} \neq 0$ under $q^*$, whereas it is zero under $q'$ because all means outside neuron $j$ were set to zero.

    The covariances $\Sigma_{a_m}$ and $\Sigma_{b_m}$ are the same under $q^*$ and $q'$ by construction of $q'$. Both distributions are global minimizers, so both must satisfy \eqref{posterior cov a_m}. Applying this equation at $q^*$ and $q'$, respectively, gives
    \begin{align*}
        d + \frac{n}{\sigma^2}\left(\hat\mu_{b_m}^2 + \Sigma_{b_m}\right) = d + \frac{n}{\sigma^2}\Sigma_{b_m},
    \end{align*}
    which implies
    \begin{equation*}
        \frac{n}{\sigma^2}\hat\mu_{b_m}^2 = 0.
    \end{equation*}
    Hence, we must have $\hat\mu_{b_m} = 0$, which is a contradiction. Therefore, at most one neuron is active under $q^*$. Since $q^*$ was an arbitrary global minimizer, the result follows.
\end{proof}

\subsection{Proof of Lemma~\ref{lem:limitdistofbhahxlemma1}}\label{app:limitdistofbhahxlemma1}
\begin{proof}
    Suppose $M > (n/\sigma^2)^2 \|R\|_2^2/d$ and, for contradiction, that some neuron is active under $q^*$. Without loss of generality, assume this neuron is neuron 1. By Theorem~\ref{thm:orthogexpecAB}, it is the only active neuron. Since $q^*$ is a global minimizer, it satisfies the equations \eqref{posterior mean a_m}-\eqref{posterior cov b_m} for $m=1$. In particular, $\mu_{b_1}\neq0$, since otherwise \eqref{posterior mean a_m} gives $\mu_{a_1}=0$ and neuron 1 would be inactive. Since $\Sigma_{a_1} \succ 0$, then
    \begin{equation*}
        \frac{n}{\sigma^2}\left(\|\mu_{a_1}\|_2^2 + \tr(\Sigma_{a_1})\right) \geq 0,
    \end{equation*}
    so \eqref{posterior cov b_m} gives $\Sigma_{b_1} \leq 1/M$. Substituting \eqref{posterior mean a_m} into \eqref{posterior mean b_m} gives
    \begin{equation*}
        \mu_{b_1} = \left(\frac{n}{\sigma^2}\right)^2 \Sigma_{b_1} R \Sigma_{a_1}R^\top \mu_{b_1}.
    \end{equation*}
    Since $\mu_{b_1} \neq 0$, then
    \begin{equation}\label{eq:lem3.2}
        R\Sigma_{a_1}R^\top = \left(\frac{\sigma^2}{n}\right) ^2 \Sigma_{b_1}^{-1}.
    \end{equation}
    Since
    \begin{equation*}
        \frac{n}{\sigma^2}\left[\mu_{b_1}^2 + \Sigma_{b_1}\right] \geq 0,
    \end{equation*}
    then $d + \frac{n}{\sigma^2} \left[\mu_{b_1}^2 + \Sigma_{b_1}\right] \geq d$, hence \eqref{posterior cov a_m} gives $\Sigma_{a_1} \preceq d^{-1} I_d$. Since $\Sigma_{b_1} \leq 1/M$ and by \eqref{eq:lem3.2}, we have
    \begin{equation*}
        \left(\frac{\sigma^2}{n}\right)^2 M \leq \left(\frac{\sigma^2}{n}\right)^2\Sigma_{b_1}^{-1} = R\Sigma_{a_1}R^\top \leq \frac{1}{d}\|R\|_2^2,
    \end{equation*}
    which implies
    \begin{equation*}
        M \leq \left(\frac{n}{\sigma^2}\right)^2 \frac{1}{d}\|R\|_2^2,
    \end{equation*}
    hence a contradiction. So, no neuron is active under $q^*$. 
\end{proof}

\subsection{Proof of Theorem~\ref{thm:limitdistofbhahx}}\label{app:limitdistofbhahxproof}
\begin{proof}
    By Lemma~\ref{lem:limitdistofbhahxlemma1}, all means are zero when $M$ is sufficiently large. Substituting \eqref{posterior cov a_m} into \eqref{posterior cov b_m} with the $\mu_{a_m}$ and $\mu_{b_m}$ set to zero gives
    \begin{equation*}
    \Sigma_{b_m}=\left(M+\frac{\frac{n}{\sigma^2}d}{d+\frac{n}{\sigma^2}\Sigma_{b_m}}\right)^{-1}.
    \end{equation*}
    Writing $t = \Sigma_{b_m}$, the fixed-point equation is equivalent to 
    \begin{equation*}
        Mt + \frac{(n/\sigma^2)dt}{d+(n/\sigma^2)t} = 1.
    \end{equation*}
    The LHS is continuous, strictly increasing for $t \geq 0$ and equals $0$ at $t = 0$, and is at least 1 at $t = 1/M$. Thus, there is a unique solution in $(0, 1/M]$. Since this equation does not depend on $m$, $\Sigma_{b_m}$ is the same for all $m$. By \eqref{posterior cov a_m}, $\Sigma_{a_m}$ is also the same for all $m$.
    
    If $x = 0$, then $BA^\top x = 0$, for which the conclusion is trivial. From hereon, suppose $x \neq 0$. Under the variational posterior $q^*(A, B)$, $\E_{q^*(A, B)}(b_m a_m^\top x) = 0$. The variance of $b_ma_m^\top x$ is 
    \begin{equation*}
        \Var_{q^*(A, B)}(b_ma_m^\top x) = \E_{q^*(A, B)}(b_m^2 (a_m^\top x)^2) = \Sigma_{b_m}x^\top\Sigma_{a_m}x \coloneq \varsigma_m^2,
    \end{equation*}
    which is dependent on $M$. Define 
    \begin{equation*}
        s_M^2 = \sum_{m=1}^M \varsigma_m^2.
    \end{equation*}
    Consider Lyapunov's condition, which requires that for some $\delta > 0$, 
    \begin{equation*}
        \lim_{M\to\infty} \frac{1}{s_M^{2 + \delta}}\sum_{m=1}^M \E_{q(A, B)}[|b_ma_m^\top x|^{2 + \delta}] = 0.
    \end{equation*}
    Taking $\delta=2$, since the variational factors are Gaussian and all means are zero, under $q(A,B)$ we have
    \begin{equation*}
        b_m\sim\mathcal N(0,\Sigma_{b_m}),\qquad a_m^\top x\sim\mathcal N(0,x^\top\Sigma_{a_m}x).
    \end{equation*}
    The fourth moment of a zero-mean Gaussian is three times the square of its variance, so
    \begin{align*}
        \E_{q^*(A,B)}(b_m^4) &= 3\Sigma_{b_m}^2,\\
        \E_{q^*(A,B)}((a_m^\top x)^4) &= 3(x^\top\Sigma_{a_m}x)^2.
    \end{align*}
    Since $b_m$ and $a_m$ are independent under $q^*(A,B)$,
    \begin{align*}
        \E_{q^*(A,B)}((b_ma_m^\top x)^4)
        &= \E_{q^*(A,B)}(b_m^4)\E_{q^*(A,B)}((a_m^\top x)^4)\\
        &= 9\Sigma_{b_m}^2(x^\top\Sigma_{a_m}x)^2\\
        &= 9\varsigma_m^4,
    \end{align*}
    so, since $\Sigma_{b_m}$ and $\Sigma_{a_m}$ are the same for all $m$ (as shown at the start of the proof), so is $\varsigma_m^2=\Sigma_{b_m}x^\top\Sigma_{a_m}x$. Writing $\varsigma^2>0$ for this common value, which may depend on $M$ but not on $m$,
    \begin{equation*}
        \frac{9\sum_{m=1}^M \varsigma_m^4}{\bigl(\sum_{m=1}^M \varsigma_m^2\bigr)^2}
        = \frac{9M\varsigma^4}{(M\varsigma^2)^2} = \frac{9}{M} \to 0
        \quad\text{as } M\to\infty,
    \end{equation*}
    thus by Lyapunov's CLT, 
    \begin{equation*}
        \frac{1}{s_M}\sum_{m=1}^M b_m a_m^\top x \ind_{q^*} \mathcal N(0, 1).
    \end{equation*}   
    
    So,
    \begin{align*}
    s_M^2 &= \sum_{m =1}^M \Sigma_{b_m} x^\top \Sigma_{a_m} x\\
    &= M \Sigma_{b_m}x^\top \Sigma_{a_m} x\\
    &= M\left(\frac{n}{\sigma^2}\tr(\Sigma_{a_m}) + M\right)^{-1} x^\top \Sigma_{a_m} x\\
    &= \frac{M x^\top x}{\left(M+\frac{n}{\sigma^2}\tr(\Sigma_{a_m})\right)\left(d+\frac{n}{\sigma^2}\Sigma_{b_m}\right)}.
    \end{align*}
    Since $0<\Sigma_{b_m}\leq1/M$, \eqref{posterior cov a_m} gives $0\prec\Sigma_{a_m}\preceq I_d/d$, so $0<\tr(\Sigma_{a_m})\leq1$. Thus,
    \begin{align*}
        0 &\leq \frac{n}{\sigma^2}\Sigma_{b_m}\leq\frac{n}{M\sigma^2}\to0,\\
        0 &\leq \frac{n}{M\sigma^2}\tr(\Sigma_{a_m})\leq\frac{n}{M\sigma^2}\to0.
    \end{align*}
    Dividing the numerator and denominator in the expression for $s_M^2$ by $M$ gives
    \begin{align*}
        s_M^2 = \frac{x^\top x}{\left(1+\frac{n}{M\sigma^2}\tr(\Sigma_{a_m})\right)\left(d+\frac{n}{\sigma^2}\Sigma_{b_m}\right)}\to \frac{x^\top x}{(1+0)(d+0)}=\frac{x^\top x}{d},
    \end{align*}
    thus by Slutsky's theorem, 
    \begin{equation}\label{eq:posteriorconvBAx}
    BA^\top x \ind_{q^*} \mathcal N\left(0, \frac{x^\top x}{d}\right).
    \end{equation}
\end{proof}

\section{Supplementary Material for \texorpdfstring{Section~\refTempering}{Section}}

\subsection{Proof of Lemma \ref{lem:tempered-zero-means}}\label{app:means-threshold}
Consider the predictive means $w^\top x$, where $w = \sum_{m=1}^M \mu_{b_m} \mu_{a_m}$. 
By AM-GM and triangle inequalities, 
\begin{align*}
    \sum_{m=1}^M \bigl[d\|\mu_{a_m}\|_2^2 + M\mu_{b_m}^2\bigr]&\geq 2\sqrt{dM}\sum_{m=1}^{M} |\mu_{b_m}|\,\|\mu_{a_m}\|_2\\
    &\geq 2\sqrt{dM}\,\|w\|_2.
\end{align*}
Hold the covariances fixed and consider the variational objective computed at the current means versus when all the means are zero. From \eqref{eq:minproblemexpanded}, if $\sigma^2$ is replaced with $T\sigma^2$, the difference between these two values of $\mathcal L^{\text{VB}}$ results in 
\begin{equation*}
    \sum_{m=1}^M \left[d\|\mu_{a_m}\|_2^2 + M\mu_{b_m}^2\right] + \frac{n}{T\sigma^2} \left(\|w\|_2^2 - 2Rw\right)
    + \frac{n}{T\sigma^2} \sum_{m=1}^M \left[\mu_{b_m}^2\tr(\Sigma_{a_m}) + \Sigma_{b_m}\|\mu_{a_m}\|_2^2\right],
\end{equation*}
where the term involving $R$ comes from the data-fit terms. Since $\|w\|_2^2 \geq 0$ and $Rw \leq \|R\|_2\|w\|_2$, the first two terms of this difference add up to at least $2(\sqrt{dM} - \frac{n}{T}\|R\|_2/\sigma^2)\|w\|_2$, which is nonnegative when $T \ge \tau_*/\sqrt M$, with $\tau_*$ as in Lemma \ref{lem:tempered-zero-means}. The final sum is strictly positive whenever any variational mean is nonzero, since $\Sigma_{a_m}\succ0$ and $\Sigma_{b_m}>0$ for every $m$. Therefore, under this condition, setting all means to zero strictly reduces the objective whenever any mean is nonzero. Every global minimizer must thus have all means equal to zero.

This recovers Lemma~\ref{lem:limitdistofbhahxlemma1} for the untempered objective and shows that the means in case 2(a) of Theorem~\ref{thm:main-theorem} are zero at every width. The proof of case 2(b) in the same theorem shows that, when $T = \tau/\sqrt M$ and $\tau < \tau_*$, exactly one neuron is active for sufficiently large $M$.

\subsection{Active Neuron Identities}\label{app:activebranchidentityproof}
\begin{lemma}\label{lem:activebranchidentity} 
    Let $Q=Id$ and consider a stationary point of the tempered variational objective with $T=T_M=\tau M^{-c}$ and exactly one active neuron indexed by $m$. Then
    \begin{equation}\label{eq:active-mean-identity}
        1 = \left(\frac{nM^c}{\tau\sigma^2}\right)^2 \Sigma_{b_m}\left(R\Sigma_{a_m}R^\top\right),
    \end{equation}
    and
    \begin{equation}
            R\Sigma_{a_m}R^\top = \frac{nM^c}{\tau\sigma^2}\mu_{b_m}^2 R\Sigma_{a_m}^2R^\top + \frac{\tr(\Sigma_{a_m})}{nM^c/(\tau\sigma^2)} + \frac{M}{(nM^c/(\tau\sigma^2))^2}.
    \end{equation}
\end{lemma}
\begin{proof}
    Substituting \eqref{posterior mean a_m} into \eqref{posterior mean b_m}, with $\sigma^2$ replaced by $T_M\sigma^2$, and using $\mu_{b_m}\neq 0$ gives
    \begin{equation*}
        1 = \left(\frac{nM^c}{\tau\sigma^2}\right)^2 \Sigma_{b_m}(R\Sigma_{a_m}R^\top).
    \end{equation*}
    Also,
    \begin{equation}\label{eq:tempered-b-inverse-variance}
        \Sigma_{b_m}^{-1} = \frac{nM^c}{\tau\sigma^2}\left(\|\mu_{a_m}\|_2^2+\tr\left(\Sigma_{a_m}\right)\right)+ M,
    \end{equation}
    By \eqref{posterior mean a_m}, with $\sigma^2$ replaced by $T_M\sigma^2$ and $T_M=\tau/M^c$,
    \begin{equation*}
        \mu_{a_m} = \frac{nM^c}{\tau\sigma^2}\mu_{b_m}\Sigma_{a_m}R^\top,
    \end{equation*}
    we have
    \begin{equation*}
        \|\mu_{a_m}\|_2^2 = \left(\frac{nM^c}{\tau\sigma^2}\right)^2\mu_{b_m}^2 R\Sigma_{a_m}^2R^\top. 
    \end{equation*}
    Substituting this into \eqref{eq:tempered-b-inverse-variance} and using \eqref{eq:active-mean-identity} gives
    \begin{align*}
        \left(\frac{nM^c}{\tau\sigma^2}\right)^2 R\Sigma_{a_m}R^\top = \Sigma_{b_m}^{-1} = \left(\frac{nM^c}{\tau\sigma^2}\right)^3\mu_{b_m}^2 R\Sigma_{a_m}^2R^\top + \frac{nM^c}{\tau\sigma^2}\tr(\Sigma_{a_m})+M.
    \end{align*}
    Dividing both sides by the positive factor $\left(nM^c/(\tau\sigma^2)\right)^2$ gives
    \begin{align*}
        R\Sigma_{a_m}R^\top = \frac{nM^c}{\tau\sigma^2}\mu_{b_m}^2 R\Sigma_{a_m}^2R^\top +\frac{\tr(\Sigma_{a_m})}{nM^c/(\tau\sigma^2)}+\frac{M}{(nM^c/(\tau\sigma^2))^2},
    \end{align*}
    which is the second identity.
\end{proof}

\subsection{Proof of Theorem~\ref{thm:main-theorem}}\label{app:proof-phase-transition}
\begin{proof}
We consider three cases on $c$: $c < 1/2$, $c = 1/2$, and $c > 1/2$. By Theorem~\ref{thm:orthogexpecAB} and Remark~\ref{rem:tempered-one-active}, at most one neuron is active under $q_{T_M}^*$.

    Let $c < 1/2$. Suppose, for contradiction, that there exists some index $m$ such that $\mu_{a_m}$ and $\mu_{b_m}$ are nonzero. By Lemma~\ref{lem:activebranchidentity},
    \begin{equation*}
        1 = \left(\frac{nM^c}{\tau\sigma^2}\right)^2 \Sigma_{b_m}(R\Sigma_{a_m}R^\top).
    \end{equation*}
    Also,
    \begin{equation}\label{eq:tempered-b-variance-bound}
        \Sigma_{b_m} = \left(\frac{nM^c}{\tau\sigma^2}\left(\|\mu_{a_m}\|_2^2 + \tr(\Sigma_{a_m})\right) + M\right)^{-1} \leq \frac{1}{M},
    \end{equation}
    so
    \begin{equation*}
        \left(\frac{nM^c}{\tau\sigma^2}\right)^2 \Sigma_{b_m} \leq \frac{n^2}{\tau^2\sigma^4}M^{2c-1}.
    \end{equation*}
    As well,
    \begin{equation}\label{eq:tempered-a-covariance-bound}
        \Sigma_{a_m} = \frac{I_d}{d+\frac{nM^c}{\tau\sigma^2}\left(\mu_{b_m}^2+\Sigma_{b_m}\right)} \preceq \frac{1}{d}I_d,
    \end{equation}
    which implies
    \begin{equation*}
        R\Sigma_{a_m}R^\top \leq \frac{\|R\|_2^2}{d}.
    \end{equation*}
    Then,
    \begin{equation*}
        1 \leq \frac{n^2}{\tau^2\sigma^4}M^{2c-1}\frac{\|R\|_2^2}{d}.
    \end{equation*}
    Since $c < \frac12$, we have $2c-1<0$, and hence
    \begin{equation*}
        \lim_{M\to\infty}\frac{n^2}{\tau^2\sigma^4}M^{2c-1}\frac{\|R\|_2^2}{d} = 0.
    \end{equation*}
    This is a contradiction. Therefore, for sufficiently large $M$, we must have $\mu_{a_m}=0 \text{ and } \mu_{b_m}=0$ for all $m$.

    It follows that for every $m$,
    \begin{equation*}
        \Sigma_{a_m} = \frac{I_d}{d+\frac{nM^c}{\tau\sigma^2}\Sigma_{b_m}},
    \end{equation*}
    and
    \begin{equation*}
        \Sigma_{b_m} = \left(\frac{nM^c}{\tau\sigma^2}\tr(\Sigma_{a_m}) + M\right)^{-1}.
    \end{equation*}
    Let $s_m^{(M)} = M\Sigma_{b_m}$. Then
    \begin{equation*}
        s_m^{(M)} = \left(1 + \frac{n}{\tau\sigma^2}M^{c-1}\tr(\Sigma_{a_m})\right)^{-1},
    \end{equation*}
    and
    \begin{equation*}
        \Sigma_{a_m} = \frac{I_d}{d+\frac{n}{\tau\sigma^2}M^{c-1}s_m^{(M)}}.
    \end{equation*}
    Since $c<1$, we have $M^{c-1}\to0$. Also, since $n/(\tau\sigma^2) M^{c-1} s_m^{(M)} > 0$, then $0 \prec \Sigma_{a_m} \preceq I_d/d$, thus,
    \begin{align*}
        0\leq1-s_m^{(M)} =\frac{\frac{n}{\tau\sigma^2}M^{c-1}\tr(\Sigma_{a_m})}{1+\frac{n}{\tau\sigma^2}M^{c-1}\tr(\Sigma_{a_m})}\leq\frac{n}{\tau\sigma^2}M^{c-1}\to0,
    \end{align*}
    so $s_m^{(M)}\to1$. Since $0<s_m^{(M)}\leq1$, we also have
    \begin{equation*}
        0\leq\frac{n}{\tau\sigma^2}M^{c-1}s_m^{(M)}\leq\frac{n}{\tau\sigma^2}M^{c-1}\to0.
    \end{equation*}
    The denominator in the expression for $\Sigma_{a_m}$ therefore tends to $d$, giving
    \begin{equation}\label{eq:zero-mean-a-covariance-limit}
        \Sigma_{a_m} \to \frac{1}{d}I_d.
    \end{equation}
    Hence,
    \begin{equation}\label{eq:zero-mean-b-variance-limit}
        \Sigma_{b_m} = \frac{s_m^{(M)}}{M} \sim \frac{1}{M}.
    \end{equation}
        We note that \eqref{eq:zero-mean-a-covariance-limit} and \eqref{eq:zero-mean-b-variance-limit} are valid for any $c<1$ and we will use these inequalities repeatedly throughout the proof.
    Thus, the tempered posterior has the same asymptotic covariance structure as in the prior case. If $x=0$, then $BA^\top x=0$ and the conclusion is immediate. For $x\neq0$, the terms $b_ma_m^\top x$ are independent under $q_{T_M}^*$ and have zero means. Define
    \begin{equation*}
        \varsigma_m^2=\Var_{r_T}(b_ma_m^\top x)=\Sigma_{b_m}x^\top\Sigma_{a_m}x,\qquad s_M^2=\sum_{m=1}^M\varsigma_m^2.
    \end{equation*}
    The bounds above do not depend on $m$, so $s_m^{(M)}\to1$ and $\Sigma_{a_m}\to I_d/d$ uniformly over $m\in\{1,\ldots,M\}$. Hence,
    \begin{align*}
        \left|s_M^2-\frac{x^\top x}{d}\right|=\left|\frac1M\sum_{m=1}^M\left(s_m^{(M)}x^\top\Sigma_{a_m}x-\frac{x^\top x}{d}\right)\right|\leq\max_{1\leq m\leq M}\left|s_m^{(M)}x^\top\Sigma_{a_m}x-\frac{x^\top x}{d}\right|\to0.
    \end{align*}
    Since the variational factors are independent zero-mean Gaussians,
    \begin{align*}
        \E_{q_{T_M}^*}((b_ma_m^\top x)^4)
        &=\E_{q_{T_M}^*}(b_m^4)\E_{q_{T_M}^*}((a_m^\top x)^4)\\
        &=9\Sigma_{b_m}^2(x^\top\Sigma_{a_m}x)^2=9\varsigma_m^4.
    \end{align*}
    Also, $\Sigma_{b_m}\leq1/M$ and $\Sigma_{a_m}\preceq I_d/d$ give $\varsigma_m^2\leq x^\top x/(Md)$. Taking $\delta=2$, the Lyapunov ratio therefore satisfies
    \begin{align*}
        \frac{1}{s_M^4}\sum_{m=1}^M\E_{q_{T_M}^*}(|b_ma_m^\top x|^4)
        &=\frac{9\sum_{m=1}^M\varsigma_m^4}{s_M^4}\\
        &\leq\frac{9\max_{1\leq m\leq M}\varsigma_m^2}{s_M^2}\\
        &\leq\frac{9x^\top x}{Md\,s_M^2}\to0,
    \end{align*}
    since $s_M^2\to x^\top x/d>0$. Thus, by Lyapunov's CLT,
    \begin{equation*}
        \frac{1}{s_M}\sum_{m=1}^M b_ma_m^\top x\ind_{q_{T_M}^*}\mathcal N(0,1).
    \end{equation*}
    Since $s_M\to\sqrt{x^\top x/d}$, Slutsky's theorem gives
    \begin{equation*}
        BA^\top x \ind_{q_{T_M}^*} \mathcal N\left(0, \frac{x^\top x}{d}\right).
    \end{equation*}

    Let $c = 1/2$ and suppose $\tau \geq \tau_*$. Suppose, for contradiction, that there exists some index $m$ such that $\mu_{a_m}$ and $\mu_{b_m}$ are nonzero. The inequalities in \eqref{eq:tempered-b-variance-bound} and \eqref{eq:tempered-a-covariance-bound} are strict since $\Sigma_{a_m}\succ0$ and $\Sigma_{b_m}>0$. Combining these bounds with \eqref{eq:active-mean-identity} at $c=1/2$ gives
    \begin{equation*}
        1 < \frac{n^2}{\tau^2\sigma^4}\frac{\|R\|_2^2}{d}.
    \end{equation*}
    However, the assumption on $\tau$ implies
    \begin{equation*}
        \frac{n^2}{\tau^2\sigma^4}\frac{\|R\|_2^2}{d}\leq 1,
    \end{equation*}
    which is a contradiction. Therefore, we must have $\mu_{a_m}=0$ and $\mu_{b_m}=0$ for all $m$. Since $c<1$, the covariance asymptotics in \eqref{eq:zero-mean-a-covariance-limit} and \eqref{eq:zero-mean-b-variance-limit} apply, and hence, as above,
    \begin{equation*}
        BA^\top x \ind_{q_T^*} \mathcal N\left(0,\frac{x^\top x}{d}\right).
    \end{equation*}

    Now suppose $\tau < \tau_*$. We first show that not all means can be zero for sufficiently large $M$. Suppose they are zero. Without loss of generality, we consider the first neuron and write $\Sigma_{a_1}=v_aI_d$, $\Sigma_{b_1}=v_b$, and $L_M=n/(T_M\sigma^2)$. By \eqref{posterior cov a_m} and \eqref{posterior cov b_m}, with $\sigma^2$ replaced by $T_M\sigma^2$ and $T_M=\tau/M^{1/2}$, we have
    \begin{equation*}
        v_a^{-1}=d+L_Mv_b,\qquad v_b^{-1}=M+L_Mdv_a.
    \end{equation*}
    By \eqref{eq:zero-mean-a-covariance-limit} and \eqref{eq:zero-mean-b-variance-limit}, $v_a\to1/d$ and $Mv_b\to1$. Since $0<\tau<\tau_*$, we have $R\neq0$. Let $(a,b)\in\R^2$. Hold all covariances and the other means fixed, and vary $\mu_{a_1}=aR^\top/\|R\|_2$ and $\mu_{b_1}=b$. Using $Q=I_d$, the terms of \eqref{eq:minproblemexpanded}, with $\sigma^2$ replaced by $T_M\sigma^2$, that depend on $(a,b)$ are
    \begin{align*}
        (d+L_Mv_b)a^2+(M+L_Mdv_a)b^2-2L_M\|R\|_2ab+L_Ma^2b^2 =\frac{a^2}{v_a}+\frac{b^2}{v_b}-2L_M\|R\|_2ab+L_Ma^2b^2.
    \end{align*}
    The Hessian with respect to $(a,b)$ at $(0,0)$ is
    \begin{equation}\label{eq:zero-mean-hessian}
        2\begin{pmatrix}
            v_a^{-1} & -L_M\|R\|_2\\
            -L_M\|R\|_2 & v_b^{-1}
        \end{pmatrix}.
    \end{equation}
    The determinant is
    \begin{align*}
        \det\left(2\begin{pmatrix}
            v_a^{-1} & -L_M\|R\|_2\\
            -L_M\|R\|_2 & v_b^{-1}
        \end{pmatrix}\right)
        &= \frac{2}{v_a}\frac{2}{v_b}-\left(-2L_M\|R\|_2\right)^2\\
        &= \frac{4}{v_av_b}-4L_M^2\|R\|_2^2\\
        &= \frac{4}{v_av_b}\left(1-L_M^2\|R\|_2^2v_av_b\right).
    \end{align*}
    For sufficiently large $M$, the determinant is negative since
    \begin{equation*}
        L_M^2\|R\|_2^2v_av_b
        \to\frac{n^2\|R\|_2^2}{d\tau^2\sigma^4}
        =\left(\frac{\tau_*}{\tau}\right)^2>1.
    \end{equation*}
    Thus, the objective decreases in some direction, contradicting the global minimality of $q_{T_M}^*$, so there must be at least one neuron that is active. Since at most one neuron is active, exactly one is for sufficiently large $M$. For each sufficiently large $M$, let $m^*=m^*(M, q_{T_M}^*)$ denote the index of the unique active neuron. Let
    \begin{equation*}
        \gamma_M = \frac{nM^{1/2}}{\tau\sigma^2}\mu_{b_{m^*}}^2
    \end{equation*}
    and
    \begin{equation*}
        t_M = \frac{nM^{1/2}}{\tau\sigma^2}\left(\mu_{b_{m^*}}^2+\Sigma_{b_{m^*}}\right).
    \end{equation*}
    Since $\Sigma_{b_{m^*}} \leq \frac{1}{M}$,
    \begin{equation*}
        0 \leq \frac{nM^{1/2}}{\tau\sigma^2}\Sigma_{b_{m^*}} \leq \frac{n}{\tau\sigma^2}M^{-1/2} \to 0,
    \end{equation*}
    hence
    \begin{equation*}
        t_M - \gamma_M \to 0.
    \end{equation*}
    Also,
    \begin{equation*}
        \Sigma_{a_{m^*}} = \frac{I_d}{t_M+d}.
    \end{equation*}
    By Lemma~\ref{lem:activebranchidentity},
    \begin{equation}\label{eq:active-covariance-identity}
        R\Sigma_{a_{m^*}}R^\top = \gamma_M R\Sigma_{a_{m^*}}^2R^\top + \frac{\tr{\Sigma_{a_{m^*}}}}{nM^{1/2}/(\tau\sigma^2)} + \frac{\tau^2\sigma^4}{n^2}.
    \end{equation}
    We first show that $(\gamma_M)_M$ is bounded. If $\gamma_M \to \infty$ along some subsequence, then $t_M \to \infty$ as well, so
    \begin{equation*}
        \Sigma_{a_{m^*}} = \frac{I_d}{t_M+d} \to 0.
    \end{equation*}
    Since $t_M\geq\gamma_M$, along this subsequence,
    \begin{align*}
        0\leq\gamma_M R\Sigma_{a_{m^*}}^2R^\top =\frac{\gamma_M\|R\|_2^2}{(t_M+d)^2} \leq\frac{\gamma_M\|R\|_2^2}{(\gamma_M+d)^2}\leq\frac{\|R\|_2^2}{\gamma_M}\to0.
    \end{align*}
    Consequently,
    \begin{equation*}
        R\Sigma_{a_{m^*}}R^\top \to 0, \qquad \gamma_M R\Sigma_{a_{m^*}}^2R^\top \to 0, \qquad \frac{\tr(\Sigma_{a_{m^*}})}{nM^{1/2}/(\tau\sigma^2)} \to 0.
    \end{equation*}
    Taking limits along this subsequence in the identity from Lemma~\ref{lem:activebranchidentity} gives
    \begin{equation*}
        0=0+0+\frac{\tau^2\sigma^4}{n^2}=\frac{\tau^2\sigma^4}{n^2}>0,
    \end{equation*}
    a contradiction. Hence, $(\gamma_M)_M$ is bounded.

    Let $\gamma_{M_\ell} \to \alpha_0$ along some subsequence. Since $t_{M_\ell} - \gamma_{M_\ell}\to 0$, we also have $t_{M_\ell}\to \alpha_0$, and therefore
    \begin{equation*}
        \Sigma_{a_{m^*}}^{(\ell)} \to A_0 \coloneq \frac{I_d}{\alpha_0+d}.
    \end{equation*}
    Taking limits in \eqref{eq:active-covariance-identity} along this subsequence gives
    \begin{equation*}
        RA_0R^\top = \alpha_0RA_0^2R^\top + \frac{\tau^2\sigma^4}{n^2}.
    \end{equation*}
    Since $(\alpha_0+d)A_0 = I_d$, we have
    \begin{equation*}
        \alpha_0A_0^2 + dA_0^2 = A_0,
    \end{equation*}
    so
    \begin{equation*}
        RA_0R^\top - \alpha_0RA_0^2R^\top = dRA_0^2R^\top.
    \end{equation*}
    Therefore,
    \begin{equation*}
        \frac{d\|R\|_2^2}{(\alpha_0+d)^2} = \frac{\tau^2\sigma^4}{n^2}.
    \end{equation*}
    Since $\alpha_0\geq0$ and $d>0$, taking the positive square root gives
    \begin{equation*}
        \alpha_0+d=\frac{n\sqrt{d}\|R\|_2}{\tau\sigma^2}=\frac{d\tau_*}{\tau},
    \end{equation*}
    and hence
    \begin{equation*}
        \alpha_0=d\left(\frac{\tau_*}{\tau}-1\right)>0,
    \end{equation*}
    since $\tau<\tau_*$. Thus, every convergent subsequence of $(\gamma_M)_M$ has the same limit $\alpha_0$, and hence
    \begin{equation*}
        \gamma_M \to \alpha_0.
    \end{equation*}

    Finally,
    \begin{equation*}
        \mu_{a_{m^*}}\mu_{b_{m^*}} = \frac{nM^{1/2}}{\tau\sigma^2}\Sigma_{a_{m^*}}R^\top\mu_{b_{m^*}}^2 = \frac{\gamma_M}{t_M+d}R^\top \to \frac{\alpha_0}{\alpha_0+d}R^\top.
    \end{equation*}

    For this active neuron,
    \begin{align}\label{eq:active-neuron-variance-bound}
        \Var_{q_{T_M}^*}(b_{m^*}a_{m^*}^\top x)
        &=\mu_{b_{m^*}}^2x^\top\Sigma_{a_{m^*}}x
        +\Sigma_{b_{m^*}}(\mu_{a_{m^*}}^\top x)^2
        +\Sigma_{b_{m^*}}x^\top\Sigma_{a_{m^*}}x \leq \frac{2}{L_M}x^\top x,
    \end{align}
    recalling $L_M=\frac{nM^c}{\tau\sigma^2}$. The bound follows from \eqref{posterior cov a_m} and \eqref{posterior cov b_m} with $\sigma^2$ replaced by $T_M\sigma^2$, together with Cauchy-Schwarz, since 
    \begin{align*}
        (\mu_{b_{m^*}}^2+\Sigma_{b_{m^*}})x^\top\Sigma_{a_{m^*}}x &\leq \frac{1}{L_M}x^\top x,\\
        \Sigma_{b_{m^*}}(\mu_{a_{m^*}}^\top x)^2 &\leq \Sigma_{b_{m^*}}
        (\|\mu_{a_{m^*}}\|_2^2)x^\top x\\ &\leq \frac{1}{L_M}x^\top x.
    \end{align*}
    Adding these two bounds gives the stated inequality.
    Thus, the active branch contributes vanishing variance. Since $c<1$, the zero branches have the same asymptotic covariance structure as in the prior case by \eqref{eq:zero-mean-a-covariance-limit} and \eqref{eq:zero-mean-b-variance-limit}. Therefore, by the same CLT argument as above,
    \begin{equation*}
        BA^\top x \ind_{r_T} \mathcal N\left(\frac{\alpha_0}{\alpha_0+d}Rx, \frac{x^\top x}{d}\right),
    \end{equation*}
    proving the assertion, replacing $\alpha_0=d\left(\frac{\tau_*}{\tau}-1\right)$.

    Now let $c>1/2$. We consider three subcases: $1/2 < c< 1$, $c = 1$, and $c > 1$. If all means were zero, write the first neuron's covariances as $\Sigma_{a_1}=v_aI_d$ and $\Sigma_{b_1}=v_b$. With $L_M=n/(T_M\sigma^2)$ again, \eqref{posterior cov a_m} and \eqref{posterior cov b_m} with $\sigma^2$ replaced by $T_M\sigma^2$ give $v_a^{-1}=d+L_Mv_b$ and $v_b^{-1}=M+L_Mdv_a$. Since $Mv_b\leq1$,
    \begin{align*}
        (v_av_b)^{-1}
        =dM+L_M Mv_b+L_Md\leq dM+L_M(d+1)=o(L_M^2).
    \end{align*}
    Hence, $L_M^2\|R\|_2^2v_av_b\to\infty$, so the Hessian in \eqref{eq:zero-mean-hessian} has negative determinant for sufficiently large $M$, ruling out all means being zero. Therefore, exactly one neuron is active. For each sufficiently large $M$, we again let $m^*=m^*(M, q_{T_M}^*)$ denote the index of the unique active neuron. Let
    \begin{equation*}
        L_M=\frac{nM^c}{\tau\sigma^2}, \qquad
        \gamma_M=L_M\mu_{b_{m^*}}^2, \qquad
        \eta_M=L_M\Sigma_{b_{m^*}}, \qquad
        t_M=\gamma_M+\eta_M.
    \end{equation*}
    Then
    \begin{equation}\label{eq:active-covariance-scale}
        \Sigma_{a_{m^*}}=\frac{I_d}{t_M+d}.
    \end{equation}
    By the first identity in Lemma~\ref{lem:activebranchidentity},
    \begin{equation*}
        \eta_M=\frac{1}{L_MR\Sigma_{a_{m^*}}R^\top}.
    \end{equation*}
    Since
    \begin{equation*}
        R\Sigma_{a_{m^*}}R^\top= \frac{\|R\|_2^2}{t_M+d},
    \end{equation*}
    we have
    \begin{equation*}
        0<\eta_M= \frac{t_M+d}{L_M\|R\|_2^2}.
    \end{equation*}
    Suppose, for contradiction, that $(\gamma_M)_M$ has a bounded subsequence, so for some subsequence $(M_\ell)_\ell$, $\gamma_{M_\ell} \leq C$ for some $C > 0$. Then, 
    \begin{equation*}
        t_{M_\ell} = \gamma_{M_\ell} + \eta_{M_\ell} \leq C + \frac{t_{M_\ell}  + d}{L_{M_\ell} \|R\|_2^2}.
    \end{equation*}
    The above inequality implies that 
    \begin{equation*}
        \left(1 - \frac{1}{L_{M_\ell} \|R\|_2^2}\right) t_{M_\ell} \leq C + \frac{d}{L_{M_\ell}\|R\|_2^2}.
    \end{equation*}
    For $M_\ell$ large enough, $ 1 / (L_{M_\ell} \|R\|_2^2) < 1/2$, so 
    \begin{equation*}
        t_{M_\ell} \leq \frac{C + \frac{d}{L_{M_\ell}\|R\|_2^2}}{1 - \frac{1}{L_{M_\ell}\|R\|_2^2}} < 2 \left(C + \frac{d}{L_{M_\ell} \|R\|_2^2}\right),
    \end{equation*}
    thus $(t_{M_\ell})_\ell$ is bounded. Since $L_{M_\ell}\to\infty$, the bound above on $\eta_M$ gives $\eta_{M_\ell}\to0$. By Bolzano-Weierstrass, $\gamma_{M_\ell}$ converges to some $\gamma\geq0$ along a further subsequence. Since $t_M=\gamma_M+\eta_M$, we also have $t_M\to\gamma$ along this subsequence.
    \begin{equation*}
        \gamma_M\to\gamma,\qquad t_M\to\gamma,
    \end{equation*}
    along this subsequence. Substituting $t_M\to\gamma$ into \eqref{eq:active-covariance-scale} gives
    \begin{equation*}
        \Sigma_{a_{m^*}}=\frac{I_d}{t_M+d}\to A_\gamma\coloneq\frac{I_d}{\gamma+d}.
    \end{equation*}
    Since $c>1/2$,
    \begin{equation*}
        \frac{\tr(\Sigma_{a_{m^*}})}{L_M}\to0
    \end{equation*}
    and
    \begin{equation*}
        \frac{M}{L_M^2}=\frac{\tau^2\sigma^4}{n^2}M^{1-2c}\to0.
    \end{equation*}
    Taking limits in the second identity in Lemma~\ref{lem:activebranchidentity} gives
    \begin{equation}\label{eq:active-limit-identity}
        RA_\gamma R^\top=\gamma RA_\gamma^2R^\top.
    \end{equation}
    However, notice that
    \begin{align*}
        A_\gamma - \gamma A_\gamma^2 &= A_\gamma(A_\gamma^{-1} - \gamma I_d)A_\gamma\\
        &= A_\gamma (dI_d) A_\gamma\\
        &= dA_\gamma^2,
    \end{align*}
    thus
    \begin{equation*}
        RA_\gamma R^\top-\gamma RA_\gamma^2 R^\top=dRA_\gamma^2R^\top>0,
    \end{equation*}
    which contradicts \eqref{eq:active-limit-identity}, since that equation gives $RA_\gamma R^\top-\gamma RA_\gamma^2R^\top=0$. Thus, $(\gamma_M)_M$ cannot have a bounded subsequence, and hence
    \begin{equation*}
        \gamma_M\to\infty.
    \end{equation*}
    Since $t_M=\gamma_M+\eta_M$ and $\eta_M>0$, we have $t_M\geq\gamma_M$, so $t_M\to\infty$, and
    \begin{equation*}
        0\leq\frac{\eta_M}{t_M}
        \leq\frac{1+d/t_M}{L_M\|R\|_2^2}\to0.
    \end{equation*}
    Hence,
    \begin{equation*}
        \frac{\gamma_M}{t_M}\to1.
    \end{equation*}
    Therefore,
    \begin{align*}
        \mu_{a_{m^*}}\mu_{b_{m^*}}=\frac{\gamma_M}{t_M+d}R^\top=\frac{\gamma_M}{t_M}\left(1+\frac{d}{t_M}\right)^{-1}R^\top\to R^\top.
    \end{align*}
    Recall that $q_{T_M}^*$ is a global minimizer of the objective in \eqref{eq:minproblemexpanded}, with $\sigma^2$ replaced by $T_M\sigma^2$, over the Gaussian mean-field family at width $M$ and temperature $T_M=\tau/M^c$. Thus,
    \begin{equation}\label{eq:active-expectation-limit}
        \E_{q_{T_M}^*}(BA^\top x)=\mu_{b_{m^*}}\mu_{a_{m^*}}^\top x\to x^\top R^\top=Rx.
    \end{equation}

    Suppose $1/2<c<1$. Equations \eqref{eq:zero-mean-a-covariance-limit} and \eqref{eq:zero-mean-b-variance-limit} show that the zero branches have the same asymptotic covariance structure as in the prior case since $c< 1$, while the active branch contributes vanishing variance. Therefore, using the same CLT argument,
    \begin{equation*}
        BA^\top x\ind_{q_{T_M}^*}\mathcal N\left(Rx,\frac{x^\top x}{d}\right).
    \end{equation*}

    Let $c=1$. For each inactive neuron $m\neq m^*$,
    \begin{align*}
        \Sigma_{a_m} &= \frac{I_d}{d+\frac{nM}{\tau\sigma^2}\Sigma_{b_m}},\\
        \Sigma_{b_m} &= \left(\frac{nM}{\tau\sigma^2}\tr{\Sigma_{a_m}}+M\right)^{-1}.
    \end{align*}
    Let $s_m^{(M)}=M\Sigma_{b_m}$. Then
    \begin{align*}
        s_m^{(M)}
        &=\left(1+\frac{n}{\tau\sigma^2}\tr{\Sigma_{a_m}}\right)^{-1}\\
        &=\left(1+\frac{nd}{ns_m^{(M)}+\tau\sigma^2d}\right)^{-1}.
    \end{align*}
    For all $m\neq m^*$, the fixed-point equation for $s_m^{(M)}$ is equivalent to
    \begin{equation}\label{eq:inactive-variance-quadratic}
        n\left(s_m^{(M)}\right)^2+\left[\tau\sigma^2d+n(d-1)\right]s_m^{(M)}-\tau\sigma^2d=0.
    \end{equation}
    The leading coefficient is positive, and the constant term is negative, so \eqref{eq:inactive-variance-quadratic} has exactly one positive root. Since the coefficients do not depend on $m$ or $M$ and $s_m^{(M)}>0$, each $s_m^{(M)}$ equals the same positive value $s$, satisfying
    \begin{equation*}
        s=\left(1+\frac{nd}{ns+\tau\sigma^2d}\right)^{-1}=\Bigl(1+\frac{\alpha d}{\alpha s+\tau d}\Bigr)^{-1}.
    \end{equation*}
    Therefore,
    \begin{equation*}
        \Sigma_{a_m}=\frac{I_d}{d+ns/(\tau\sigma^2)}.
    \end{equation*}
    If $x=0$, then the sum of the contributions from inactive neurons is zero and the conclusion is immediate. Suppose $x\neq0$. Under $q_{T_M}^*$, the $M-1$ terms $b_ma_m^\top x$ with $m\neq m^*$ are independent and have zero means. Since $\Sigma_{b_m}=s/M$, for each inactive neuron,
    \begin{equation*}
        \varsigma_m^2\coloneq\Var_{q_{T_M}^*}(b_ma_m^\top x)=\Sigma_{b_m}x^\top\Sigma_{a_m}x=\frac{s}{M}\frac{x^\top x}{d+ns/(\tau\sigma^2)}=\frac{1}{M}\frac{x^\top x}{n/(\tau\sigma^2)+d/s}>0.
    \end{equation*}
    Define the total variance of the sum over inactive neurons by
    \begin{equation*}
        s_M^2\coloneq\sum_{\substack{m=1\\m\neq m^*}}^M\varsigma_m^2=\frac{M-1}{M}\frac{x^\top x}{n/(\tau\sigma^2)+d/s}\to\frac{x^\top x}{n/(\tau\sigma^2)+d/s}>0.
    \end{equation*}
    Since $b_m$ and $a_m^\top x$ are independent zero-mean Gaussians,
    \begin{equation*}
        \E_{q_{T_M}^*}(b_m^4)=3\Sigma_{b_m}^2,\qquad\E_{q_{T_M}^*}((a_m^\top x)^4)=3(x^\top\Sigma_{a_m}x)^2.
    \end{equation*}
    Hence,
    \begin{equation*}
        \E_{q_{T_M}^*}((b_ma_m^\top x)^4)=\E_{q_{T_M}^*}(b_m^4)\E_{q_{T_M}^*}((a_m^\top x)^4)=9\Sigma_{b_m}^2(x^\top\Sigma_{a_m}x)^2=9\varsigma_m^4.
    \end{equation*}
    The variances $\varsigma_m^2$ are the same for all inactive neurons, so taking $\delta=2$, the Lyapunov ratio is
    \begin{equation*}
        \frac{1}{s_M^4}\sum_{\substack{m=1\\m\neq m^*}}^M\E_{q_{T_M}^*}(|b_ma_m^\top x|^4)=\frac{9\displaystyle\sum_{\substack{m=1\\m\neq m^*}}^M\varsigma_m^4}{\left(\displaystyle\sum_{\substack{m=1\\m\neq m^*}}^M\varsigma_m^2\right)^2}=\frac{9}{M-1}\to0.
    \end{equation*}
    Thus, Lyapunov's CLT gives
    \begin{equation*}
        \frac{1}{s_M}\sum_{\substack{m=1\\m\neq m^*}}^M b_ma_m^\top x\ind_{q_{T_M}^*}\mathcal N(0,1).
    \end{equation*}
    Since $s_M\to\sqrt{x^\top x/(n/(\tau\sigma^2)+d/s)}$, Slutsky's theorem gives
    \begin{equation*}
        \sum_{\substack{m=1\\m\neq m^*}}^M b_ma_m^\top x\ind_{q_{T_M}^*}\mathcal N\left(0,\frac{x^\top x}{n/(\tau\sigma^2)+d/s}\right).
    \end{equation*}
    Since the active branch contributes vanishing variance, it follows that
    \begin{equation*}
        BA^\top x\ind_{q_{T_M}^*}
        \mathcal N\left(Rx,\frac{x^\top x}{n/(\tau\sigma^2)+d/s}\right).
    \end{equation*}

    Let $c>1$. The expectation converges to $Rx$ by \eqref{eq:active-expectation-limit}, so it remains to show that the variance converges to zero. For each inactive neuron $m\neq m^*$, both means are zero, and \eqref{posterior cov a_m} with $\sigma^2$ replaced by $T_M\sigma^2$ gives
    \begin{equation*}
        \Sigma_{b_m}\Sigma_{a_m}=\frac{\Sigma_{b_m}}{d+L_M\Sigma_{b_m}}I_d\preceq\frac{1}{L_M}I_d.
    \end{equation*}
    Since the $M-1$ inactive contributions are independent under $q_{T_M}^*$, their total variance satisfies
    \begin{align*}
        \Var_{q_{T_M}^*}\left(\sum_{\substack{m=1\\m\neq m^*}}^M b_ma_m^\top x\right)&=\sum_{\substack{m=1\\m\neq m^*}}^M\Sigma_{b_m}x^\top\Sigma_{a_m}x\leq\frac{M-1}{L_M}x^\top x\leq\frac{M}{L_M}x^\top x=\frac{\tau\sigma^2}{n}M^{1-c}x^\top x\to0.
    \end{align*}
    By \eqref{eq:active-neuron-variance-bound}, the active neuron also satisfies
    \begin{equation*}
        \Var_{q_{T_M}^*}(b_{m^*}a_{m^*}^\top x)\leq\frac{2}{L_M}x^\top x=\frac{2\tau\sigma^2}{n}M^{-c}x^\top x\to0.
    \end{equation*}
    Therefore, independence of the active and inactive contributions gives
    \begin{equation*}
        \Var_{q_{T_M}^*}(BA^\top x)\leq\frac{M+1}{L_M}x^\top x=\frac{\tau\sigma^2}{n}(M+1)M^{-c}x^\top x\to0.
    \end{equation*}
    Together with \eqref{eq:active-expectation-limit}, Chebyshev's inequality gives
    \begin{equation*}
        BA^\top x\ind_{q_{T_M}^*}Rx,
    \end{equation*}
    as required.
\end{proof}

\subsection{Proof of Corollary~\ref{cor:recovery}}\label{app:proof-recovery}
\begin{proof}
    We recall that $\alpha > 0$, $R \neq 0$, and under the assumption $Q = I_d$,
    \begin{equation*}
        m_{\text{NNGP}}(x) = \frac{\alpha}{\alpha + d}Rx, \qquad v_{\text{NNGP}} = \frac{x^\top x}{\alpha + d}. 
    \end{equation*}
    In particular, since $R\neq0$, then $x=R^\top$ satisfies $Rx=\|R\|_2^2>0$, so $m_{\text{NNGP}}(R^\top)=\frac{\alpha}{\alpha+d}\|R\|_2^2$ lies strictly between $0$ and $RR^\top=\|R\|_2^2$. Likewise, for $x\neq0$, $v_{\text{NNGP}}(x)$ lies strictly between $0$ and the prior variance $v_0(x)=x^\top x/d$. A limiting mean equal to $0$ or to $Rx$, or a limiting variance equal to $0$ or to $v_0(x)$, therefore does not match the NNGP posterior for all $x$.

    For 1., let $c = 1/2$. If $\tau \geq \tau_*$, then the limiting mean is zero by case~2(a) of Theorem~\ref{thm:main-theorem}, so it does not agree with $m_{\text{NNGP}}(x)$ at $x = R^\top$. If $\tau < \tau_*$, then case~2(b) gives $m_{1/2}(x) = (1-\tau/\tau_*)Rx$. Since $R \neq 0$, this equals $m_{\text{NNGP}}(x)$ for all $x$ if and only if
    \begin{equation*}
        1-\frac{\tau}{\tau_*} = \frac{\alpha}{\alpha+d},
    \end{equation*}
    which gives $\tau = \tau_*d/(\alpha+d) = \tau_\mu$. Since $\tau_\mu < \tau_*$, this value satisfies the condition of case~2(b).

    For 2., let $c = 1$ and $s \in (0,1)$ be the unique solution of \eqref{eq:s-fixed-point}. By case~4 of Theorem~\ref{thm:main-theorem}, the limiting variance is $v_{\text{NNGP}}(x)$ for all $x$ if and only if
    \begin{equation*}
        \frac{\alpha}{\tau}+\frac{d}{s} = \alpha+d.
    \end{equation*}
    Letting $D = \tau(\alpha+d)-\alpha$, this is equivalent to $d/s = D/\tau$. Since $s \in (0,1)$, we must have $D > \tau d$, which is equivalent to $\tau > 1$. Conversely, if $\tau > 1$, then
    \begin{equation*}
        \tilde s = \frac{\tau d}{D} \in (0,1).
    \end{equation*}
    Therefore, by uniqueness of $s$, the variances agree for all $x$ if and only if $\tau > 1$ and $\tilde s$ satisfies \eqref{eq:s-fixed-point}.

    Rewrite the fixed-point equation as
    \begin{equation*}
        \frac{1}{s}-1 = \frac{\alpha d}{\alpha s+\tau d}.
    \end{equation*}
    Since $\alpha+D = \tau(\alpha+d)$, substituting $\tilde s$ gives
    \begin{equation*}
        \frac{1}{\tilde s}-1 = \frac{\alpha(\tau-1)}{\tau d}, \qquad \alpha\tilde s+\tau d = \frac{\tau d(\alpha+D)}{D} = \frac{\tau^2d(\alpha+d)}{D}.
    \end{equation*}
    Thus, $\tilde s$ satisfies the fixed-point equation if and only if
    \begin{align*}
        \frac{\alpha(\tau-1)}{\tau d} &= \frac{\alpha D}{\tau^2(\alpha+d)},\\
        \tau(\tau-1)(\alpha+d) &= d\bigl(\tau(\alpha+d)-\alpha\bigr),
    \end{align*}
    or equivalently, $g(\tau) = 0$, where
    \begin{equation*}
        g(\tau) = \tau^2-(d+1)\tau+\frac{d\alpha}{\alpha+d}.
    \end{equation*}
    Since $g(1) = -d^2/(\alpha+d) < 0$, this quadratic has one root below $1$ and one above $1$. The latter is $\tau_v$, so the limiting variance agrees with $v_{\text{NNGP}}$ if and only if $\tau = \tau_v$ (the second root). Moreover, $g(d+1) = d\alpha/(\alpha+d) > 0$, so $\tau_v \in (1,d+1)$.

    For 3., Theorem~\ref{thm:main-theorem} gives limiting variance $v_0(x)$ for $c < 1$ and zero for $c > 1$. Since $v_{\text{NNGP}}(x)$ lies strictly between these values for $x \neq 0$, recovering the NNGP posterior variance requires $c = 1$. The theorem also gives limiting mean zero for $c < 1/2$ and $Rx$ for $c > 1/2$. Both differ from $m_{\text{NNGP}}(x)$ at $x = R^\top$, so recovering the NNGP posterior mean requires $c = 1/2$. Therefore, no schedule $T_M = \tau/M^c$ recovers both moments for all $x$.
\end{proof}

\section{Computing the Variational Bayes Free Energy Objective}\label{app:longcomputationLVB}
Since we assume the variational posterior factorizes in \eqref{eq:vbfactorization} and impose priors $p(A)$ and $p(B)$ on $A$ and $B$ respectively, the free energy $F_{\text{VB}}(q)$ from Section~\ref{sec:variational-optimization} can be rewritten as 
\begin{equation}\label{eq:vbfreeenergy}
    F_{\text{VB}}(q) = \E_{q(A, B)} \left[\log \frac{q(A)q(B)}{p(Y^n \mid X^n, A, B)p(A)p(B)}\right].
\end{equation}
Here $\E_q$ denotes expectation with respect to $q(A,B)$. Since $q\in \mathcal{Q}_{MF}$, \eqref{eq:vbfreeenergy} can be decomposed as 
\begin{equation}\label{eq:vbfreeexpanded}
    F_{\text{VB}}(q) = \E_q (-\log p(Y^n \mid X^n, A, B)) + \sum_{m=1}^M \left[\E_q\left(\log \frac{q(a_m)}{p(a_m)}\right) + \E_q\left(\log \frac{q(b_m)}{p(b_m)}\right)\right].
\end{equation}
By the normality assumptions on the distributions $q(a_m), p(a_m), q(b_m), p(b_m)$, 
\begin{align}
    \E_q\left(\log \frac{q(a_m)}{p(a_m)}\right) &= \frac12 \left[d\left(\|\mu_{a_m}\|_2^2 + \tr(\Sigma_{a_m})\right) - \log |\Sigma_{a_m}| - d \log d - d\right] \nonumber\\
    &= \frac12 \left[d\left(\|\mu_{a_m}\|_2^2 + \tr(\Sigma_{a_m})\right) - \log |\Sigma_{a_m}| \right]  +C \label{eq:klr(a_m)phi(a_m)},
\end{align}
and similarly 
\begin{equation}\label{eq:klr(b_m)phi(b_m)}
    \E_q\left(\log \frac{q(b_m)}{p(b_m)}\right) = \frac 12 \left[M\left(\mu_{b_m}^2 + \Sigma_{b_m}\right) - \log \Sigma_{b_m}\right] + C.
\end{equation}
By the normality assumption of $p(Y^n \mid X^n, A, B)$, 
\begin{equation*}
    -\log p(Y^n \mid X^n, A, B) = \frac 1{2\sigma^2} \sum_{i=1}^n \left\|y_i - \sum_{m=1}^M b_m a_m^\top x_i\right\|_2^2 + C.
\end{equation*}
Let $f_i = \sum_{m=1}^M b_m a_m^\top x_i$, and $u_{im}=b_m a_m^\top x_i-\mu_{b_m}\mu_{a_m}^\top x_i$, so $\E_q[u_{im}]=0$. Then
\begin{align}
    \E_q\left(\left\|y_i - \sum_{m=1}^M b_m a_m^\top x_i\right\|_2^2\right) &= \E_q\left[\|y_i - \E_q(f_i)\|_2^2 + \|f_i - \E_q(f_i)\|_2^2 \right] \nonumber \\
    &= \left\|y_i - \sum_{m=1}^M \mu_{b_m} \mu_{a_m}^\top x_i\right\|_2^2 + \E_q\left[\left\|\sum_{m=1}^M b_m a_m^\top x_i - \sum_{m=1}^M \mu_{b_m} \mu_{a_m}^\top x_i \right\|_2^2\right] \nonumber \\
    &= \left\|y_i - \sum_{m=1}^M \mu_{b_m} \mu_{a_m}^\top x_i\right\|_2^2 + \E_q\left(\left\|\sum_{m=1}^M u_{im}\right\|_2^2\right) \nonumber\\
    &= \left\|y_i - \sum_{m=1}^M \mu_{b_m} \mu_{a_m}^\top x_i\right\|_2^2 + \E_q\left(\sum_{m=1}^M \|u_{im}\|_2^2 + \sum_{m \neq j} u_{im}u_{ij}\right).\nonumber
\end{align}
By the assumed factorization $q(A,B)=\prod_{m=1}^M q(a_m)q(b_m)$, the centered contributions from distinct neurons are independent, hence 
\begin{equation*}
    \E_q(u_{im} u_{ij}) = \E_q(u_{im})\E_q(u_{ij}) = 0
\end{equation*}
for all $m \neq j$. So, 
\begin{equation}\label{eq:expecneglogposterior}
    \E_q\left( \sum_{i=1}^n \left\|y_i - \sum_{m=1}^M b_m a_m^\top x_i\right\|_2^2\right) = \sum_{i=1}^n\left\|y_i - \sum_{m=1}^M \mu_{b_m} \mu_{a_m}^\top x_i\right\|_2^2 + \sum_{i=1}^n \E_q\left(\sum_{m=1}^M \|u_{im}\|_2^2\right).
\end{equation}
For each $m \in \{1,\ldots, M\}$, 
\begin{align}
    \E_q(\|u_{im}\|_2^2 ) &= \E_q (\|b_m a_m^\top x_i - \mu_{b_m} \mu_{a_m}^\top x_i\|_2^2)\nonumber \\
    &= \E_q(\|b_m a_m^\top x_i\|_2^2) - \|\mu_{b_m} \mu_{a_m}^\top x_i\|_2^2\nonumber \\
    &= \E_q((a_m^\top x_i)^2 b_m^2) - \|\mu_{b_m} \mu_{a_m}^\top x_i \|_2^2 \nonumber\\
    &= ((\mu_{a_m}^\top x_i)^2 + x_i^\top \Sigma_{a_m} x_i) (\mu_{b_m}^2 + \Sigma_{b_m}) - \|\mu_{b_m}\mu_{a_m}^\top x_i\|_2^2 \nonumber\\
    &= \mu_{b_m}^2 x_i ^\top \Sigma_{a_m}x_i + \Sigma_{b_m}(\mu_{a_m}^\top x_i)^2 + \Sigma_{b_m}x_i^\top \Sigma_{a_m} x_i \label{eq:meanUh}.
\end{align}
\eqref{eq:klr(a_m)phi(a_m)}, \eqref{eq:klr(b_m)phi(b_m)}, \eqref{eq:expecneglogposterior}, and \eqref{eq:meanUh} come together to rewrite $\mathcal L^{\text{VB}}(q)$ as
\begin{multline*}
   \mathcal L^{\text{VB}}(q) = \sum_{m=1}^M\left[d\left(\|\mu_{a_m}\|_2^2 + \tr(\Sigma_{a_m})\right) - \log |\Sigma_{a_m}| + M\left(\mu_{b_m}^2 + \Sigma_{b_m}\right) - \log \Sigma_{b_m}\right] \\
    + \frac{1}{\sigma^2}\sum_{i=1}^n \left[\left\|y_i - \sum_{m=1}^M \mu_{b_m}\mu_{a_m}^\top x_i\right\|_2^2 + \sum_{m=1}^M \left(\mu_{b_m}^2 x_i ^\top \Sigma_{a_m}x_i + \Sigma_{b_m}(\mu_{a_m}^\top x_i)^2 + \Sigma_{b_m}x_i^\top \Sigma_{a_m} x_i \right)\right]+C.
\end{multline*}

\section{Likelihood Tempering as KL-Divergence Downweighting}\label{app:tempdownweight}
Let $\pi(A,B \mid X^n,Y^n)$ be the true posterior where $(X^n,Y^n)$ are the observed variables and $(A,B)$ are the latent ones. Recall the KL-divergence between the variational posterior $q(A,B)$ and the true posterior $\pi(A,B \mid X^n,Y^n)$ 
\begin{equation*}
    \text{KL}  (q(A,B)  \|  \pi(A,B \mid X^n,Y^n) ) = \int q(A,B) \log \frac{q(A,B)}{\pi(A,B \mid X^n,Y^n)}   \, dA\,dB .
\end{equation*}
By Bayes' theorem,
\small
\begin{multline*}
    \text{KL}(q(A,B)  \|  \pi(A,B \mid X^n,Y^n) ) = \int q(A,B) \log \frac{q(A,B)}{p(A,B)}  \, dA\,dB  + \int q(A,B) \log p(Y^n \mid X^n)  \, dA\,dB \\- \int q(A,B) \log p(Y^n \mid X^n,A,B)  \, dA\,dB .
\end{multline*}
\normalsize
Notice that the first quantity is the KL-divergence between the variational posterior $q$ and the true prior $p(A,B)$, the second is equal to the log evidence $\log p(Y^n \mid X^n)$ since $q$ is a density, and the last quantity is the expectation of $\log p(Y^n \mid X^n,A,B)$ with respect to $q$, denoted as $\E_q[\log p(Y^n \mid X^n,A,B)]$. Recall the evidence lower bound (ELBO)
\begin{equation*}
\text{ELBO}(q) = \E_q  [\log p(Y^n \mid X^n,A,B) ] - \text{KL} (q(A,B)  \|  p(A,B) ).
\end{equation*}
Since the log evidence is constant with respect to $q$, the problem of variational inference where we want to minimize the KL-divergence between densities $q$ from a family of variational distributions $\mathcal{Q}$ and the true posterior $\pi(A,B \mid X^n,Y^n)$ is equivalent to maximizing the ELBO over $q$, thus the variational inference problem can be expressed as 
\begin{equation*}
    \min_{q \in \mathcal{Q}} \text{KL} (q(A,B)  \|  p(A,B) ) - \E_q [\log p(Y^n \mid X^n,A,B) ].
\end{equation*}
Define the tempered normalizing constant
\begin{equation*}
    Z_T(X^n,Y^n)=\int p(A,B)p(Y^n\mid X^n,A,B)^{1/T}\,dA\,dB.
\end{equation*}
Then
\begin{align*}
    &\text{KL}(q(A,B) \| \pi_T(A,B \mid X^n,Y^n)) = \int q(A,B) \log \frac{q(A,B) Z_T(X^n,Y^n)}{p(A,B) p(Y^n \mid X^n,A,B)^{1/T}} \,dA\,dB \\
    &\qquad= \text{KL}(q(A,B) \| p(A,B)) + \log Z_T(X^n,Y^n) - \frac{1}{T}\E_q[\log p(Y^n \mid X^n,A,B)],
\end{align*}
so the variational inference problem is 
\begin{equation}\label{eq:tempered-variational-objective}
    \min_{q \in \mathcal{Q}} \text{KL} (q(A,B)  \|  p(A,B) ) - \frac{1}{T}\E_q[\log p(Y^n \mid X^n,A,B)], 
\end{equation}
which is equivalent to multiplying the KL term  $\text{KL} (q(A,B) \| p(A,B) )$ by $T$.
\begin{equation*}
    \min_{q \in \mathcal{Q}} T \text{KL}(q(A,B)\|p(A,B)) - \E_q[\log p(Y^n \mid X^n,A,B)].
\end{equation*}

For the Gaussian observation model in Section~\ref{sec:lbnn}, the log-likelihood is
\begin{equation}\label{eq:gaussian-log-likelihood-tempering}
    \log p(Y^n\mid X^n,A,B)=-\frac{n}{2}\log(2\pi\sigma^2)-\frac{1}{2\sigma^2}\sum_{i=1}^n\|y_i-BA^\top x_i\|_2^2.
\end{equation}
Taking expectations in \eqref{eq:gaussian-log-likelihood-tempering} and multiplying by $-1/T$ gives
\begin{equation*}
    -\frac{1}{T}\E_q[\log p(Y^n\mid X^n,A,B)]=\frac{1}{2T\sigma^2}\E_q\left[\sum_{i=1}^n\|y_i-BA^\top x_i\|_2^2\right]+\frac{n}{2T}\log(2\pi\sigma^2).
\end{equation*}
The last term is independent of $q$. Thus, the objective in \eqref{eq:tempered-variational-objective} is the untempered objective with $\sigma^2$ replaced by $T\sigma^2$, up to a term independent of $q$. Since $T\sigma^2>0$, Theorem~\ref{thm:orthogexpecAB} applies, so all but at most one neuron have zero variational means under any global minimizer of the tempered objective. This proves Remark~\ref{rem:tempered-one-active}.

\section{Background Details}\label{app:background-details}

\subsection{Deep Feed-Forward Bayesian Neural Networks}\label{app:bnns}
A feed-forward network with input dimension $M_0=d$, hidden widths $M_1,\ldots,M_L$ and scalar output, $M_{L+1}=1$, computes $f_\theta(x)=f^{L+1}(x)$, where $h^0(x)=x$ and
\begin{align*}
    f^\ell(x) &= \frac{W_\ell h^{\ell-1}(x)}{\sqrt{M_{\ell-1}}}, && \ell=1,\ldots,L+1,\\
    h^\ell(x) &= \varphi\bigl(f^\ell(x)\bigr), && \ell=1,\ldots,L,
\end{align*}
with $W_\ell\in\R^{M_\ell\times M_{\ell-1}}$, $\theta=(W_\ell)_{\ell=1}^{L+1}$ and an activation $\varphi$ applied componentwise. With independent standard Gaussian entries in each $W_\ell$, the factor $1/\sqrt{M_{\ell-1}}$ keeps the pre-activations of order one as the widths grow, which is the NTK parametrization of \citet{jacotNeuralTangentKernel2018}. The network of Section~\ref{sec:lbnn} is the case $L=1$, $M_1=M$ and $\varphi=\mathrm{id}$, with $A=W_1^\top/\sqrt d$ and $B=W_2/\sqrt M$.

\subsection{Variational Inference}\label{app:VI}
For a model with prior $p(\theta)$, likelihood $p(\mathcal D\mid\theta)$ and posterior $\pi(\theta\mid\mathcal D)$, every density $q$ satisfies
\begin{equation*}
    \log p(\mathcal D)=\text{ELBO}(q)+\text{KL}\bigl(q\,\|\,\pi(\cdot\mid\mathcal D)\bigr),
\end{equation*}
with $\text{ELBO}(q)=\E_q[\log p(\mathcal D\mid\theta)]-\text{KL}(q\,\|\,p)$. Since the evidence $p(\mathcal D)$ does not depend on $q$, minimising the KL-divergence to the posterior over a family $\mathcal{Q}$ is equivalent to maximising the ELBO, which is a lower bound on $\log p(\mathcal D)$ \citep{bleiVariationalInferenceReview2017}. A mean-field family consists of the densities $q(\theta)=\prod_{j=1}^Jq_j(\theta_j)$ for a fixed partition $\theta=(\theta_1,\ldots,\theta_J)$. The family $\mathcal{Q}_{\text{MF}}$ of Section~\ref{sec:lnnvi} uses the blocks $a_1,\ldots,a_M,b_1,\ldots,b_M$ with Gaussian factors, so that the coordinates within each $a_m$ remain dependent.

\subsection{Coordinate-Ascent Updates and First-Order Conditions}\label{app:coordinateupdates}
Let $q\in\mathcal{Q}_{\text{MF}}$, and let $q^-_{a_m}$ and $q^-_{b_m}$ denote the products of all its factors except $q(a_m)$ and $q(b_m)$, respectively. As a function of $q(a_m)$, the free energy \eqref{eq:vbfreeenergyoriginal} equals $\text{KL}(q(a_m)\,\|\,\tilde q(a_m))$ up to an additive constant, and similarly for $q(b_m)$, where
\begin{align}
    \tilde q(a_m) &\propto p(a_m)\exp\bigl(\E_{q^-_{a_m}}[\log p(Y^n\mid X^n,A,B)]\bigr), \label{eq:ram}\\
    \tilde q(b_m) &\propto p(b_m)\exp\bigl(\E_{q^-_{b_m}}[\log p(Y^n\mid X^n,A,B)]\bigr) \label{eq:rbm}
\end{align}
are the coordinate-ascent updates applied iteratively in practical algortihms \citep{bleiVariationalInferenceReview2017}. Both targets are Gaussian, as computed below, and the KL-divergence between two Gaussian distributions is stationary in the parameters of its first argument only where the two coincide. The stationary points of \eqref{eq:minproblemexpanded} are therefore the fixed points $q(a_m)=\tilde q(a_m)$ and $q(b_m)=\tilde q(b_m)$, $m=1,\ldots,M$.

Since $Q=I_d$, one has $\sum_ix_ix_i^\top=nI_d$ and $\sum_iy_ix_i=nR^\top$. As the factors are independent, the expected log-likelihoods in \eqref{eq:ram} and \eqref{eq:rbm} are, up to additive constants,
\begin{align*}
    &-\frac\alpha2\bigl(\mu_{b_m}^2+\Sigma_{b_m}\bigr)\|a_m\|_2^2+\alpha\mu_{b_m}a_m^\top\Bigl(R-\sum_{k\neq m}\mu_{b_k}\mu_{a_k}^\top\Bigr)^\top,\\
    &-\frac\alpha2\bigl(\|\mu_{a_m}\|_2^2+\tr(\Sigma_{a_m})\bigr)b_m^2+\alpha b_m\Bigl(R-\sum_{k\neq m}\mu_{b_k}\mu_{a_k}^\top\Bigr)\mu_{a_m},
\end{align*}
with $\alpha=n/\sigma^2$. Completing the squares with $$p(a_m)\propto\exp\left(-\frac{d\|a_m\|_2^2}{2}\right),$$ and $$p(b_m)\propto\exp\left(-\frac{Mb_m^2}{2}\right),$$ shows that both targets are Gaussian, with the covariances \eqref{posterior cov a_m} and \eqref{posterior cov b_m} and the means
\begin{align*}
    \mu_{a_m} &= \alpha\mu_{b_m}\Sigma_{a_m}\Bigl(R-\sum_{k\neq m}\mu_{b_k}\mu_{a_k}^\top\Bigr)^\top,\\
    \mu_{b_m} &= \alpha\Sigma_{b_m}\Bigl(R-\sum_{k\neq m}\mu_{b_k}\mu_{a_k}^\top\Bigr)\mu_{a_m}.
\end{align*}
At a stationary point with at most one active neuron, the residual may be replaced by $R$ for every $m$, which gives \eqref{posterior mean a_m} and \eqref{posterior mean b_m}: the other means vanish for the active neuron, and both sides vanish for an inactive one. Replacing $\alpha$ by $\alpha/T$ gives the first-order conditions of the tempered objective.

\section{Simulation Details}\label{app:simulation-details}
\subsection{Setup}
Here, we study the finite-width behaviour of tempered variational posterior under the linear neural network model and Gaussian priors introduced in Section~\ref{sec:lbnn}. We use the same model throughout and vary the network width $M$ according to the tempering schedules described below. 

We generate one dataset with sample size $n = 40$ and input dimension $d = 2$, and start with an $n\times d$ matrix of independent $\mathcal N(0, 1)$ entries. Then, we center each column, orthonormalize the columns using a QR decomposition, and multiply the resulting matrix by $\sqrt n$. As a result, the design matrix $X$ satisfies 
\begin{equation*}
    Q = \frac 1n X^\top X = I_d. 
\end{equation*}
The responses $y_i$ are generated by 
\begin{equation*}
    y_i = 1.4 x_{i1} - 0.8 x_{i2} + \varepsilon_i, \quad \varepsilon_i\iidsim \mathcal N(0,\sigma^2), \quad \sigma^2=20.
\end{equation*}
When fitting the variational posterior, we treat $\sigma^2$ as a known entity. The same dataset is used for every network width and tempering regime, so the simulation studies the effect of increasing width conditional on a fixed dataset. 

We consider widths of 
\begin{equation*}
    M \in \{2^4, 2^6,\ldots, 2^{28}\}
\end{equation*}
and set $T_M = \tau/M^c$. Let 
\begin{equation*}
    R = \frac 1n y^\top X, \qquad \alpha = \frac{n}{\sigma^2}, \qquad \tau_* = \frac{\alpha \|R\|_2}{\sqrt d},
\end{equation*}
where $\tau_*$ is the threshold appearing in Lemma \ref{lem:tempered-zero-means}. Additionally, define 
\begin{equation*}
    \tau_\mu = \tau_* \left( 1 + \frac \alpha d\right)^{-1}
\end{equation*}
and 
\begin{equation*}
    \tau_v = \frac{d + 1 + \sqrt{(d+1)^2 - \frac{4d\alpha}{\alpha + d}}}{2}.
\end{equation*}
For this dataset, $\tau_*=2.908$, $\tau_\mu=1.454$ and $\tau_v=2.618$.

The eight tempering schedules we consider are
\begin{align*}
    (c,\tau)\in\{&
    (1/4,\tau_\mu), (1/2,\tau_*/0.7), (1/2,\tau_\mu), (3/4,\tau_\mu),\\
    &
    (0.95,\tau_v), (1,\tau_v), (1.05,\tau_v), (5/4,\tau_\mu)
    \}.
\end{align*}
The constants are each computed once and held fixed as only $M$ varies. When $c = 1/2$, the value $\tau_\mu$ is chosen so that the limiting variational posterior expectation agrees with the untempered NNGP posterior expectation, while $\tau_*/0.7$ gives a value above the critical threshold. When $c = 1$, the value $\tau_v$ is chosen so that the limiting variational posterior variance agrees with the untempered NNGP posterior variance. We also use $\tau_v$ for $c = 0.95$ and $c = 1.05$ to emphasize the effect of changing $c$ near $c=1$. 

\subsection{Variational Posterior Computation}\label{sec:simulation-computation}
For every $M$ and $T_M$, we compute a Gaussian mean-field variational posterior using the representation from Section~\ref{sec:lbnn}. Since at most one neuron is active under a global minimizer (Remark~\ref{rem:tempered-one-active}), we assume without loss of generality that it is the first, and write
\begin{equation*}
    q_T(a_1) = \mathcal N_d(\mu_a, v_a I_d),  \quad q_T(b_1) = \mathcal N(\mu_b, v_b).
\end{equation*}
For $m = 2,\ldots, M$, write 
\begin{equation*}
    q_T(a_m) = \mathcal N_d(0, v_{a0} I_d),  \quad q_T(b_m) = \mathcal N(0, v_{b0}).
\end{equation*}
Let 
\begin{equation*}
    L_M = \frac{n}{T_M\sigma^2}. 
\end{equation*}
Since $Q = I_d$, the first-order equations for the first neuron are 
\begin{align}
    v_a &= \left[d+L_M(\mu_b^2+v_b)\right]^{-1}, \,\mu_a = L_Mv_a\mu_bR^\top, \label{eq:a-first-neuron-eq}\\
    v_b &= \left[M+L_M(\|\mu_a\|_2^2+dv_a)\right]^{-1}, \,\mu_b = L_Mv_bR\mu_a. \label{eq:b-first-neuron-eq}
\end{align}

For $m = 2,\ldots, M$, set $\mu_{a_m} = 0$, $\mu_{b_m} = 0$. The common variances of these neurons are then obtained by iterating 
\begin{equation*}
    v_{a0} = \left(d + L_M v_{b0}\right)^{-1}, \quad v_{b0} = \left(M + L_M dv_{a0}\right)^{-1}. 
\end{equation*}
We initialize this iteration at the prior variances: $v_{a0} = 1/d$ and $v_{b0} = 1/M$.

For the first neuron, we initialize $\mu_b=M^{-1/4}$ and set $v_a$ and $v_b$ equal to the fitted zero-mean variances. The nonzero initialization allows the iteration to move away from the zero-mean stationary point when the data support nonzero variational means. We use the same update procedure for every tempering regime, rather than substituting the asymptotic expressions from Theorem~\ref{thm:main-theorem}.

We take inspiration from \cite{nakajimaGlobalAnalyticSolution2013} and consider an iterated conditional modes (ICM) algorithm approach \citep{besagStatisticalAnalysisDirty1986,bishopPatternRecognitionMachine2006}. At each iteration, we first update $v_a$ and $\mu_a$ using the current values of $v_b$ and $\mu_b$, and then update $v_b$ and $\mu_b$ using the new values of $v_a$ and $\mu_a$. This means that the equations in \eqref{eq:b-first-neuron-eq} hold after each iteration. We then check if the updated values of $v_b$ and $\mu_b$ also satisfy the equations in \eqref{eq:a-first-neuron-eq}. The covariance residual is computed as 
\begin{equation*}
    e_{\text{cov}} = \left|v_a\left[d+L_M(\mu_b^2+v_b)\right]-1\right|
\end{equation*}
and expectation residual as 
\begin{equation*}
    e_{\text{mean}} = \frac{\|\mu_a - L_Mv_a\mu_bR^\top\|_\infty}{\frac{1}{\sqrt d} + \|\mu_a\|_\infty}.
\end{equation*}
For the remaining neurons, where $\mu_{a_m} = 0$ and $\mu_{b_m} = 0$, we use the covariance residual. 

We continue iterating until either all residuals are smaller than $10^{-10}$, or a maximum of $2\times 10^7$ iterations is reached. If the residual tolerance is not met within this limit, the script stops with an error. Therefore, every accepted fit numerically satisfies the finite-width first-order equations.

\subsection{Posterior Mean and Variance}\label{sec:posterior-simulation}
After fitting the variational posterior, we compute the mean and variance of the latent network output at
\begin{equation*}
    x^* = \frac{1}{\sqrt 2}(1,-1)^\top.
\end{equation*}
Since the variational factors are independent,
\begin{equation*}
    \E_{q_{T_M}}(f_\theta(x^*)) = \mu_b\mu_a^\top x^*,
\end{equation*}
and
\begin{equation*}
    \Var_{q_{T_M}}(f_\theta(x^*)) = (M-1)v_{b0}v_{a0}(x^*)^\top x^* + \mu_b^2v_a(x^*)^\top x^*+ v_b(\mu_a^\top x^*)^2 + v_bv_a(x^*)^\top x^*.
\end{equation*}
We evaluate these expressions using the fitted variational parameters for each width and tempering schedule. They give the exact moments of the fitted finite-width distribution. Since we consider the latent network output, observation noise is not included.

We plot the predictive mean and variance for each of the eight tempering schedules against $M$ on a logarithmic horizontal axis. The mean plot includes the prior mean
\begin{equation*}
    m_0(x^*) = 0,
\end{equation*}
the untempered NNGP posterior mean 
\begin{equation*}
    m_{\text{NNGP}}(x^*) = \frac{\alpha}{\alpha + d}Rx^*,
\end{equation*}
and the least squares prediction 
\begin{equation*}
    m_\infty(x^*) = Rx^*.
\end{equation*}
The variance plot includes the prior variance 
\begin{equation*}
    v_0(x^*) = \frac{(x^*)^\top x^*}{d} = \frac12,
\end{equation*}
the untempered NNGP posterior variance 
\begin{equation*}
    v_{\text{NNGP}}(x^*) = \frac{(x^*)^\top x^*}{\alpha + d} = \frac 14,
\end{equation*}
and the zero-variance point-mass limit. These quantities are only included in the plots as theoretical reference values. The mean and variance curves are obtained from the fitted finite-width variational posteriors. The simulation script uses base R with seed $1001$ for dataset generation.

\subsection{Verification of the Global Minimizer}
Every fitted distribution is also a global minimizer. For the fits with zero means, $T_M\geq \tau_*/\sqrt M$, so Lemma \ref{lem:tempered-zero-means} applies. For the remaining fits, eliminating the means from \eqref{eq:a-first-neuron-eq} and \eqref{eq:b-first-neuron-eq} gives
\begin{equation*}
    \frac{M}{L_M}u^2+(d-1)u-L_Md\|R\|_2^2=0,
    \qquad u=v_a^{-1}.
\end{equation*}
This equation has exactly one positive root, so there is only one stationary point with one active neuron, apart from changing the signs of both means and choosing which neuron is active. By Theorem~\ref{thm:orthogexpecAB}, the only other possible global minimizer is the zero-mean stationary point (which is unique as proved in Appendix \ref{app:limitdistofbhahxproof} for $T=1$, and valid for $T>0$ by the change from $\sigma^2$ to $T\sigma^2$). For each of these fits, the tempered objective at that point is at least $4.04$ greater than at the fitted distribution on the scale of $\mathcal L^{\text{VB}}$.

\subsection{Simulation Infrastructure}\label{app:sim-infrastructure}
The simulation script uses base R (version 4.6.1) with seed $1001$ for dataset generation. All computations were run on a laptop CPU (Apple M3 Pro), and the full simulation takes under a minute; no GPU or computing cluster was used.

\end{document}